\documentclass{article} 
\usepackage{iclr2026_conference_lla,times}

\usepackage{hyperref}
\usepackage{url}

\usepackage{MyPreamble}
\usepackage{amsmath,amsfonts,amssymb,amsthm}

\newcommand\OursFull{\underline{{\bfseries Re}}inforcement Routing for \underline{{\bfseries Mix}}ture-of-LoRAs}
\newcommand\Ours{ReMix}

\title{\Ours{}: Reinforcement Routing for Mixtures of LoRAs in LLM Finetuning}

\author{Ruizhong Qiu\thanks{Work done during an internship at Meta AI.} \\
University of Illinois\\Urbana-Champaign, IL, USA \\
\texttt{rq5@illinois.edu} \\
\And
Hanqing Zeng, Yinglong Xia, Yiwen Meng, Ren Chen \\
Meta AI\\Menlo Park, CA, USA \\
\texttt{\{zengh,yxia,ywmeng,renchen\}@meta.com} \\
\And
Jiarui Feng \\
Washington University\\St. Louis, WA, USA \\
\texttt{feng.jiarui@wustl.edu} \\
\And
Dongqi Fu \\
Meta AI\\Sunnyvale, CA, USA \\
\texttt{dongqifu@meta.com} \\
\And
Qifan Wang, Jiayi Liu, Jun Xiao, Xiangjun Fan, Benyu Zhang, Hong Li \\
Meta AI\\Menlo Park, CA, USA \\
\texttt{\{wqfcr,liujiayi,junxiao,maxfan,byzhang,hongli\}@meta.com} \\
\And
Zhining Liu, Hyunsik Yoo, Zhichen Zeng, Tianxin Wei, Hanghang Tong \\
University of Illinois\\Urbana-Champaign, IL, USA \\
\texttt{\{liu326,hy40,zhichenz,twei10,htong\}@illinois.edu} \\
\\
}

\iclrfinalcopy 
\begin{document}

\maketitle

\begin{figure*}[t]
\centering
\includegraphics[width=\linewidth]{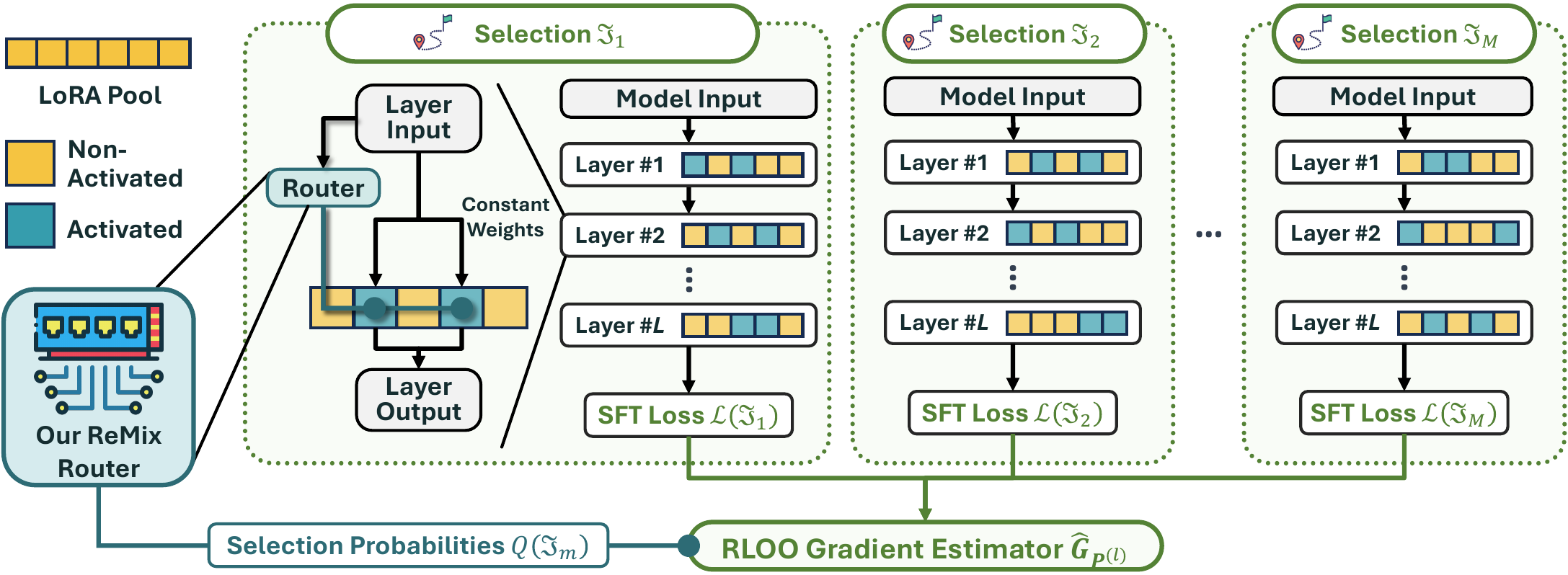}
\vspace{-0.5em}
\caption{Finetuning procedure of our proposed \Ours{}.}
\label{fig:main}
\end{figure*}

\begin{abstract}
Low-rank adapters (LoRAs)
are a parameter-efficient finetuning technique that injects trainable low-rank matrices into pretrained models to adapt them to new tasks. Mixture-of-LoRAs models expand neural networks efficiently by routing each layer input to a small subset of specialized LoRAs of the layer. Existing Mixture-of-LoRAs routers assign a learned routing weight to each LoRA to enable end-to-end training of the router. Despite their empirical promise, we discover, both theoretically and empirically, that the routing weights often collapse to only one LoRA even when we activate $k>1$ LoRAs during finetuning. 
When one LoRA has a dominantly large weight, then the computation of the other $k-1$ LoRAs are essentially \textbf{wasted} because using $k>1$ would have similar accuracy to $k=1$.
This essentially limits the number of effective LoRAs and thus severely hinders the realized expressive power of existing Mixture-of-LoRAs models.
In this work, we attribute this weakness to the nature of learnable routing weights and rethink the fundamental design of the router. To address this critical issue, we propose a simple yet effective router design that we call \emph{\OursFull{}} (\Ours). Our key idea is using \emph{non-learnable} routing weights to ensure all active LoRAs to be equally effective, with no single LoRA dominating the routing weights. However, such non-learnable routing weights make it infeasible to directly train routers via gradient descent. In response, we further propose an unbiased gradient estimator for the router and employ the reinforce leave-one-out (RLOO) technique to reduce the variance of the estimator. 
Our gradient estimator also enables to scale up training compute to boost the predictive performance of our \Ours{}. Extensive experiments demonstrate that our proposed \Ours{} significantly outperforms state-of-the-art parameter-efficient finetuning methods under a small number of activated parameters.
\end{abstract}


\addtocontents{toc}{\protect\setcounter{tocdepth}{0}} 



\begin{wrapfigure}{r}{0.4\linewidth}
\centering
\vspace{-4em}
\includegraphics[width=0.85\linewidth]{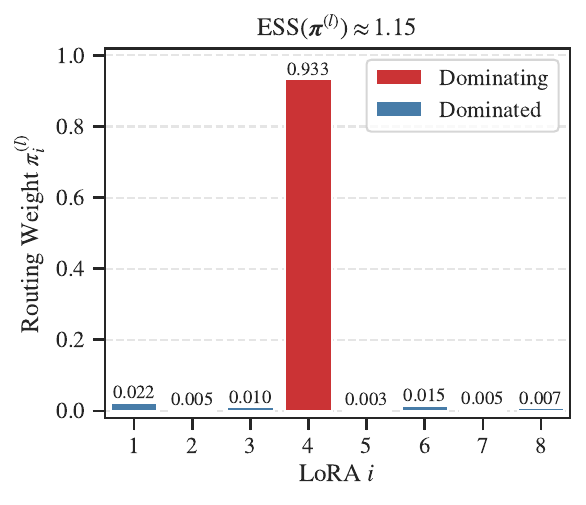}
\vspace{-0.75em}
\caption{We often observe that \textbf{only one LoRA} has a dominating routing weight that is close to one at deeper layers even when we activate $k>1$ LoRAs during finetuning. Consequently, the computation of the other $k-1$ LoRAs are essentially \textbf{wasted} because using $k>1$ would have similar accuracy to $k=1$. (The dominating LoRA can \textbf{differ on different inputs}.)}
\label{fig:ess-w}
\vspace{-1em}
\end{wrapfigure}

\section{Introduction}

Parameter-efficient fine-tuning (PEFT) aims to reduce the number of trainable parameters while achieving strong task performance (e.g., \citealp{he2022sparseadapter, ruckle2020adapterdrop, jie2023revisiting}). 
Among PEFT methods, low-rank adapters (LoRAs, \citealp{hu2021lora}) have become particularly prominent due to their simplicity and effectiveness. By injecting lightweight low-rank matrices into pretrained weight matrices, LoRAs allow downstream adaptation with a small fraction of trainable parameters, making them particularly attractive for resource-constrained settings and large-scale multi-task deployments.

Building on the success of LoRAs, researchers have proposed Mixture-of-LoRAs to further enhance parameter efficiency and expressive power (e.g., \citealp{huang2023lorahub,wang2023multilora,hydralora,smore}). The key idea is to route each input through a small pool of LoRAs per layer, thereby enabling specialization of LoRAs across different input distributions. Central to this framework is the router, which assigns routing weights across a pool of multiple LoRAs. Current approaches rely on learned routing weights, trained jointly with task objectives via gradient descent. In principle, such routers should flexibly allocate inputs across LoRAs and balance capacity usage.

Despite their empirical promise, we theoretically reveal a striking weakness of existing Mixture-of-LoRAs routers: routing weights can be extremely imbalanced, often collapsing to a very small number of LoRAs with high probability. 
Furthermore, we empirically observe that the imbalance even worsens as finetuning progresses, where the effective number of LoRAs \textbf{drops to 1 quickly} even though we activate $k>1$ LoRAs during finetuning. This essentially disables all other LoRAs, thereby limiting the realized expressive power of the mixture. When only one LoRA has a dominating weight, the computation of the other $k-1$ LoRAs are essentially \textbf{wasted} because using $k>1$ would have similar accuracy to $k=1$. We call this critical issue \emph{routing weight collapse}.

To address this critical limitation, we revisit the fundamental design of the router. Instead of relying on learned continuous weights that tend to result in routing weight collapse, we propose a simple yet effective design called \emph{\OursFull{}} (\Ours{}), which enforces a constant routing weights across all {\em activated LoRAs}. This ensures that all active LoRAs contribute equally, avoiding collapse into a single dominant LoRA. Since non-learnable weights prevent direct training via backpropagation, we reformulate the router training problem as reinforcement learning (RL), where we view the supervised finetuning loss as the negative reward and the router as the policy model of RL. We then propose an unbiased, RLOO-based gradient estimator tailored for our proposed router. This unbiased estimator enables stable training and scales efficiently to large compute budgets, unlocking the full potential of mixture-based parameter-efficient finetuning. Our main contributions are as follows.
\begin{itemize}
\item {\bf Theoretical insights on routing weight collapse}: 
We theoretically reveal and empirically observe a fundamental limitation of routers: We observe that for each given input, often only one LoRA has a dominating routing weight that is close to one. This extreme imbalance essentially disables all other LoRAs and severely limits the expressive power of the model.
\item {\bf Simple yet effective router}: To address routing imbalance, we propose a new router design with a constant routing weight across all activated LoRAs. Our design does not introduce any additional inference cost over existing Mixture-of-LoRAs methods.
\item {\bf Reinforcement learning for router training}: To address the non-differentiability of our proposed router, we reformulate the router training problem as reinforcement learning and propose an unbiased, RLOO-based gradient estimator tailored for our proposed router.
\item {\bf Empirical evaluation}: Through extensive experiments across diverse benchmarks, we demonstrate that \Ours{} consistently outperforms state-of-the-art parameter-efficient finetuning methods under the \textbf{same parameter budgets}.
\end{itemize}

\begin{wrapfigure}{r}{0.5\linewidth}
\centering
\vspace{-2.5em}
\includegraphics[width=0.8\linewidth]{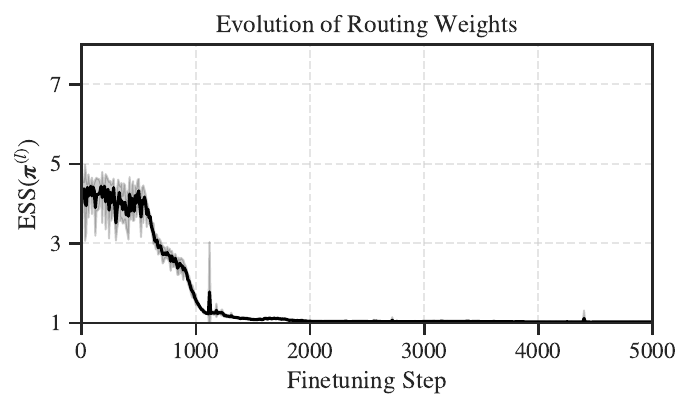}
\vspace{-1em}
\caption{Routing weight collapse even \textbf{worsens as finetuning progresses}. The effective support size $\OP{ESS}(\BM\pi^{(l)})$ often \textbf{drops to 1 quickly} during finetuning.}
\vspace{-2em}
\label{fig:ess-line}
\end{wrapfigure}

\section{Motivation: Routing Weight Collapse}
\label{sec:imb}

In this section, we analyze and reveal a critical limitation of existing Mixture-of-LoRAs routers: the extreme imbalance in routing weights assigned to different LoRAs. After introducing preliminaries in Section~\ref{ssec:prelim}, we first make a fundamental theoretical analysis showing that the number of effective LoRAs per layer is severely limited. Then, we corroborate this finding with empirical evidence from our experiments.

\subsection{Preliminaries: Mixture of LoRAs}\label{ssec:prelim}

Mixture-of-LoRAs is a type of parameter-efficient adapter that enhances the capacity of large models using only a small number of LoRAs and a lightweight router to dynamically select the LoRAs for each input. 

Let $D$ denote the hidden dimensionality of the model. Following prior work, we apply LoRAs to feedforward layers in the LLM, and all other layers are frozen. Let $\BM x^{(l)},\BM y^{(l)}\in \mathbb{R}^D$ denote the input and the output of feedforward layer $l$ ($l=1,\dots,L$), respectively.
Let $n$ denote the number of LoRAs we use in the mixture. Each LoRA $i=1,\dots,n$ is a linear map parameterized as a low-rank decomposition $\BM B_i^{(l)}\BM A_i^{(l)}\in\BB R^{D\times D}$, where $\BM A_i^{(l)} \in \mathbb{R}^{r \times D}$ and $\BM B_i^{(l)} \in \mathbb{R}^{D \times r}$ are learnable parameters, and $r\ll D$ is the rank of LoRAs. A \emph{router} of a layer $l$ is a small neural network parameterized by a matrix $\BM P^{(l)} \in \mathbb{R}^{n \times D}$ that predicts a categorical distribution over $n$ LoRAs via the softmax operation:
\AL{
\BM\pi^{(l)} := \softmax(\BM P^{(l)}\BM x^{(l)}) \in \mathbb{R}^n.
}
Here, $\pi_i^{(l)}:=(\BM\pi^{(l)})_i$ represents the routing weight assigned to the $i$-th LoRA. 
Given the routing weights, the output of a typical Mixture-of-LoRAs layer is computed as:
\AL{
\BM y^{(l)} := \BM W^{(l)}\BM x^{(l)} + \sum_{i=1}^n \pi_i^{(l)} \BM B_i^{(l)} \BM A_i^{(l)} \BM x^{(l)}.
}
where $\BM W^{(l)}\in \mathbb{R}^{D\times D}$ denotes the frozen weight of the layer $l$. This formulation intends to differentiably select a specialized subset of LoRAs for each given layer input $\BM x^{(l)}$.

\subsection{Theoretical Analysis}

We make a fundamental theoretical analysis showing that the number of effective LoRAs is severely limited. Recall that the output of a Mixture-of-LoRAs layer is a weighted sum of the LoRA outputs, where the routing weights are typically normalized via a softmax function. While this design allows for end-to-end training, we show that it introduces a strong tendency for the router to concentrate most of the weight on only one or two LoRAs.


To quantify the effective number of LoRAs, we use the \emph{effective support size} notion from information theory. For routing weights $\BM\pi^{(l)}\ne\BM0$, the effective support size (ESS) 
is defined as \citep{grendar2006entropy}
\AL{\OP{ESS}(\BM\pi^{(l)}):=\frac{\big(\sum_{i=1}^n|\pi^{(l)}_i|\big)^{\!2}}{\sum_{i=1}^n|\pi^{(l)}_i|^2}=\bigg(\frac{\|\BM\pi^{(l)}\|_1}{\|\BM\pi^{(l)}\|_2}\bigg)^{\!\!2}.}
The intuition of $\OP{ESS}(\BM\pi^{(l)})$ is that it measures the number of LoRAs with relatively large routing weights. For example, if $\BM\pi^{(l)}$ is one-hot, then we have $\OP{ESS}(\BM\pi^{(l)})=1$; if $\BM\pi^{(l)}$ is uniform over $n$ LoRAs, then we have $\OP{ESS}(\BM\pi^{(l)})=n$. Note that $\OP{ESS}(\BM\pi^{(l)})$ concerns only about the utilization of LoRAs for each given input, not the overall utilization of each LoRA over the entire dataset. With the help of this notion of ESS, we formally state our theoretical observation in the following Theorem~\ref{THM:imb}.

\begin{THM}[routing weight collapse]\label{THM:imb}
Suppose that the router parameter matrix $\BM P^{(l)}$ follows i.i.d.\ Gaussian initialization with variance $\sigma^2>0$ (e.g., Kaiming initialization, \citealp{he2015delving}).
Then for any $0<\delta<1$, with probability at least $1-\delta$, the effective support size of $\BM\pi^{(l)}$ is at most
\AM{\OP{ESS}(\BM\pi^{(l)})\le\!\left(\!1\!+\!\frac1{\exp\!\left(\!\frac{\delta\sigma\|\BM x^{(l)}\|_2}{\frac32\sqrt{\frac{\uppi}{\ln3}\ln n}+\frac1{\sqrt{2\uppi}\,2^{n-\log_2n-1}}}-\ln(n-1)\!\right)}\!\right)^{\!\!\!2}\!\!.}
\end{THM}

The proof of Theorem~\ref{THM:imb} is deferred to Appendix~\ref{app:theo-imb}. Our Theorem~\ref{THM:imb} shows that with high probability, only an extremely small number of LoRAs have relatively large routing weights. For instance, if $\sigma=1$, and there are $n=8$ LoRAs, and $\BM x^{(l)}$ is a Rademacher random vector in $\BB R^{D=1024}$, then our Theorem~\ref{THM:imb} shows that with probability at least 84.19\%, {\bfseries at most two} LoRAs have relatively large routing weights. Since each routing weight is a coefficient in front of each LoRA, a relatively small routing weight would essentially disable that LoRA. Moreover, those extremely small routing weights also vanish the gradient back-propagated to the corresponding LoRAs and consequently hinder the learning process of these LoRAs. Therefore, this phenomenon severely limits the realized expressive power and performance of the Mixture-of-LoRAs model.


\subsection{Empirical Analysis}

To further validate our theoretical result in Theorem~\ref{THM:imb}, we conduct a case study on the routing weights across LoRAs in MixLoRA \citep{mixlora}, a popular Mixture-of-LoRAs method. Specifically, we track the routing weights of the last layer throughout the training process on the GSM8K dataset (a mathematical reasoning dataset) and compute the distributions and the ESS of the routing weights.

To visualize the distribution of routing weights, we plot a typical histogram of routing weights during finetuning, shown in Figure~\ref{fig:ess-w}. We often observe that \textbf{only one LoRA} has a dominating routing weight close to one at deeper layers while \textbf{all other seven LoRAs} have negligibly small routing weights. The observation echoes our Theorem~\ref{THM:imb} that the learned routing weights are indeed extremely imbalanced. The extremely limited number of effective LoRAs severely restricts the realized expressive power of the Mixture-of-LoRAs model: when only one LoRA has a dominating weight, the computation of the other $k-1$ LoRAs are essentially \textbf{wasted} because using $k>1$ would have similar accuracy to $k=1$.

To further study how the distribution of routing weights evolve over the finetuning process, we plot the ESS of the worst routing weights at each training step, as shown in Figure~\ref{fig:ess-line}. In fact, the imbalance even worsens as the finetuning process progresses. We often observe that the effective support size $\OP{ESS}(\BM\pi^{(l)})$ often \textbf{drops to 1} quickly during finetuning. For instance, even though the ESS is around 4 at step 0, the ESS quickly decreases to 1 since only step 1000 and never increases thereafter. As a remark, different inputs can have different dominating LoRAs, but they still suffer from routing weight collapse.


These results highlight a fundamental limitation of current Mixture-of-LoRAs routers: despite the potential for increased expressivity via multiple LoRAs, the model essentially activates only an extremely small subset for each given input. This motivates our proposed method, which aims to ensure a more balanced and effective use of other available LoRAs.

\section{Simple yet Effective Method: \Ours{}}


In this section, we propose a simple yet effective router design called \emph{\OursFull{}} (\Ours{}). First, we introduce the adapter architecture in Section~\ref{ssec:method-arch}. Then, we describe the finetuning procedure in Section~\ref{ssec:method-train} and the inference procedure in Section~\ref{ssec:method-infer}.

\subsection{Adapter Architecture: Non-Learnable Weight}\label{ssec:method-arch}

In this subsection, we introduce the adapter architecture of our proposed method \Ours{}.

Given layer input $\BM x^{(l)}\!\in\!\mathbb{R}^D$, we first produce an $n$-way categorical \emph{routing distribution} $\BM q^{(l)}:=\softmax(\BM P^{(l)}\BM x^{(l)})\in\BB R_{\ge0}^{n}$ over the $n$ LoRAs, where $\BM P^{(l)}\in\BB R^{n\times D}$ denotes the learnable parameter matrix of the router. Then, we use the routing distribution $\BM q^{(l)}$ to select the $k$ LoRAs $\CAL I^{(l)}:=\{i_1^{(l)},\dots,i_k^{(l)}\}$ to activate. The LoRA selection procedure differs between finetuning and inference, which we will describe later in Sections~\ref{ssec:method-train} \& \ref{ssec:method-infer}.

To address the extreme imbalance of routing weights in existing Mixture-of-LoRAs models (Section~\ref{sec:imb}), we assign the a constant routing weight $\omega>0$ to all the $k$ activated LoRAs and zero routing weights to all non-activated LoRAs. Formally, our routing weights $\BM\pi^{(l)}$ are defined as
\AL{
\pi^{(l)}_i:=\omega\BBM1_{[i\in\CAL I^{(l)}]}=\begin{cases}
\omega,&\text{if }i\in\CAL I^{(l)},\\
0,&\text{if }i\notin\CAL I^{(l)},
\end{cases}\quad i=1,\dots,n,
}
where $\omega$ is either LoRA-type $\omega:={2}/{kr}$ \citep{hu2021lora} or rsLoRA-type $\omega:={2}/{\sqrt{kr}}$ \citep{kalajdzievski2023rank}. In fact, our method is not sensitive to $\omega$, as shown in Section~\ref{ssec:exp-omega}.

Notably, our design ensures that $\OP{ESS}(\BM\pi^{(l)})=k$, which is in stark contrast to existing learnable routing weights (Theorem~\ref{THM:imb}). Finally, we compute the layer output $\BM y^{(l)}\in\BB R^D$ as a $\BM\pi^{(l)}$-weighted sum over $k$ activated LoRAs. Using the sparse nature of our routing weights $\BM\pi^{(l)}$, the computation of layer output $\BM y^{(l)}$ can be simplified as follows:
\AL{
\BM y^{(l)} :={}&\BM W^{(l)}\BM x^{(l)} + \sum_{i=1}^n\pi^{(l)}_i\BM B_{i}^{(l)} \BM A_{i}^{(l)} \BM x^{(l)}
\BM W^{(l)}\BM x^{(l)} + \omega\sum_{j=1}^k \BM B_{i_j^{(l)}}^{(l)} \BM A_{i_j^{(l)}}^{(l)} \BM x^{(l)}.
}

\subsection{Finetuning Procedure: RLOO}\label{ssec:method-train}

In this subsection, we describe how to train our proposed \Ours{} during finetuning. 

Let $\FR I:=(\CAL I^{(1)},\dots,\CAL I^{(L)})$ denote the collection of activated LoRAs of the entire LLM for a given model input, and we call $\FR I$ a \emph{selection}. Let $\CAL L(\FR I)$ denote the supervised finetuning (SFT) loss when activated LoRAs are $\FR I$. Regarding LoRA parameters $\BM A^{(l)}_i,\,\BM B^{(l)}_i$, since the LLM output is differentiable w.r.t.\ LoRA parameters, we can simply use their gradients $\BM G_{\BM A^{(l)}_i}:=\nabla_{\BM A^{(l)}_i}\CAL L(\FR I),\,\BM G_{\BM B^{(l)}_i}:=\nabla_{\BM B^{(l)}_i}\CAL L(\FR I)$ to train them.

Regarding router parameters, however, the LLM output is not differentiable w.r.t.\ router parameters $\BM P^{(l)}$ because routing weights $\pi^{(l)}_i$ are a constant hyperparameter $\omega$. Consequently, we cannot directly compute their gradients $\nabla_{\BM P^{(l)}_i}\CAL L(\FR I)$ as it is not defined. To address this non-differentiability, we propose sampling each $\CAL I^{(l)}$ from the corresponding routing distribution $\BM q^{(l)}$ so that $\Exp_{\CAL I^{(l)}\sim\BM q^{(l)}}[\CAL L(\FR I)]$ depends on router parameters $\BM P^{(l)}$. This enables $\Exp_{\CAL I^{(l)}\sim\BM q^{(l)}}[\CAL L(\FR I)]$ to be differentiable w.r.t.\ router parameters $\BM P^{(l)}$. Hence, we propose using $\BM G_{\BM P^{(l)}}:=\nabla_{\BM P^{(l)}}\Exp_{\CAL I^{(l)}\sim\BM q^{(l)}}[\CAL L(\FR I)]$ as a \emph{surrogate gradient} of $\BM P^{(l)}$. Formally, given the routing distribution $\BM q^{(l)}:=\softmax(\BM P^{(l)}\BM x^{(l)})$, we sample $k$ LoRAs $(i^{(l)}_1,\dots,i^{(l)}_k)\sim \BM q^{(l)}$ from $\BM q^{(l)}$ without replacement to compose the activated LoRA subset $\CAL I^{(l)}:=(i^{(l)}_1,\dots,i^{(l)}_k)$, where sampling without replacement ensures that the $k$ activated LoRAs are mutually distinct.

However, due to the exponentially many possibilities of $\FR I$, it is computationally intractable to straightforwardly compute $\BM G_{\BM P^{(l)}}=\nabla_{\BM P^{(l)}}\Exp_{\CAL I^{(l)}\sim\BM q^{(l)}}[\CAL L(\FR I)]$ by definition. To address this intractability, we alternatively consider router training as a \textbf{reinforcement learning} (RL) problem, where we view the SFT loss $\CAL L(\FR I)$ as the negative reward and the routers $\BM q^{(l)}$ as the policy model. With this alternative view, we are able to employ the policy gradient estimator in RL to estimate the surrogate gradient $\BM G_{\BM P^{(l)}}$. Formally, we independently sample $M$ selections $\FR J_1,\dots,\FR J_M$, where $M$ represents the training compute budget. Write each selection as $\FR I_m=:(\CAL I^{(l)}_m)_{l=1}^L=:((i_{m,j}^{(l)})_{j=1}^k)_{l=1}^L$ ($m=1,\dots,M$), where $\CAL I^{(l)}_m$ denotes the ordered set of selected LoRAs at the $l$-th layer in the $m$-th selection $\FR J_m$, and $i_{m,j}^{(l)}$ denotes the $j$-th selected LoRA at the $l$-th layer in the $m$-th selection $\FR J_m$. Due to sampling without replacement, the probability of each selection $\FR J_m$ is
$ 
Q(\FR J_m):=\prod_{l=1}^L\prod_{j=1}^k\frac{q^{(l)}_{i^{(l)}_{m,j}}}{1-\sum_{j'=1}^{j-1}q^{(l)}_{i^{(l)}_{m,j'}}}.
$ 
Under this factorization, 
we further derive the RLOO gradient estimator \citep{kool2019buy} to estimate the surrogate gradient $\BM G_{\BM P^{(l)}}$:
\AL{
&\HAT{\BM G}_{\BM P^{(l)}}\!:=\!{}\frac1{M\!-\!1}\!\!\sum_{m=1}^M\!\!\big(\CAL L(\FR I_m)-\BAR{\CAL L}\big)\nabla_{\BM P^{(l)}}\log Q(\FR J_m)
\frac1{M\!-\!1}\!\!\sum_{m=1}^M\!\!\big(\!\CAL L(\FR I_m)-\BAR{\CAL L}\!\big)\!\sum_{j=1}^k\!\nabla_{\BM P^{(l)}}\log\frac{q^{(l)}_{i^{(l)}_{m,j}}}{1\!-\!\!\!\sum\limits_{j'=1}^{j-1}\!q^{(l)}_{i^{(l)}_{m,j'}}},\nonumber
}
where $\BAR{\CAL L}:=\frac1M\sum_{m=1}^M\CAL L(\FR I_m)$ is the average SFT loss across the $M$ selections.
It can be shown that our RLOO gradient estimator is unbiased: $\Exp_{\FR J_1,\dots,\FR J_m}[\HAT{\BM G}_{\BM P^{(l)}}]=\BM G_{\BM P^{(l)}}$.


\subsection{Inference Procedure: Top-$k$ Selection}\label{ssec:method-infer}

In this subsection, we describe how our proposed \Ours{} selects the LoRAs to activate during inference. While it is possible to randomly sample the LoRAs like the finetuning procedure, here we propose a better, theoretically optimal approach to LoRA selection.

Our following Theorem~\ref{THM:topk} shows that the optimal strategy is in fact \textbf{top-$k$ selection} as long as the router is trained sufficiently well.

\begin{THM}[optimality of top-$k$ selection]\label{THM:topk}
Let $\CAL I^{(l)*}=\{i_1^{(l)*},\dots,i_k^{(l)*}\}$ denote the optimal subset of LoRAs for a given model input. As long as the router $\BM q^{(l)}$ is trained sufficiently well such that $\Prb_{\CAL I^{(l)}\sim\BM q^{(l)}}[\CAL I^{(l)}=\CAL I^{(l)*}]>\frac12$,
then the LoRAs $i$ with top-$k$ $q_i^{(l)}$ are \textbf{guaranteed} to constitute the best subset $\CAL I^{(l)*}$: $\mathop{\RM{argtop}_k}_{i=1}^nq^{(l)}_i=\CAL I^{(l)*}$.
\end{THM}

The proof of Theorem~\ref{THM:topk} is deferred to Appendix~\ref{app:theo-topk}. Notably, our Theorem~\ref{THM:topk} shows that when sampling yields the optimal subset with probability above 
$50\%$, then top-$k$ selection substantially improves this probability to $100\%$. 
Intuitively speaking, as long as the router is trained sufficiently well, then the optimal choices of LoRAs are in fact those $i$ with top-$k$ $q_i^{(l)}$. Motivated by Theorem~\ref{THM:topk}, we employ top-$k$ LoRA selection (instead of random sampling) during inference:
\AL{
\CAL I^{(l)}=\{i_1^{(l)},\dots,i_k^{(l)}\}:=\mathop{\RM{argtop}_k}_{i=1}^nq^{(l)}_i.
}


\section{Experiments}

\begin{table*}[t]
\begin{center}
\caption{Comparison with existing parameter-efficient finetuning methods. Our \Ours{} consistently outperforms all baseline methods while maintaining strong parameter efficiency. We use the \textbf{same parameter count budget} for all methods to ensure fair comparison; for each method, we search for the best parameter count and report the accuracy under the best parameter count.
}
\label{tab:exp-main}
\vspace{-0.5em}
\definecolor{C1}{rgb}{0.84, 0.93, 0.98}%
\resizebox{\linewidth}{!}{
\begin{tabular}{ll|cc|cc|cc|cc}
\toprule
\multirow{2}*{\textbf{Type}} & \multirow{2}*{\textbf{Method}} & \multicolumn{2}{c|}{GSM8K} & \multicolumn{2}{c|}{HumanEval} & \multicolumn{2}{c|}{ARC-c}&\multicolumn{2}{c}{Average}\\
&&Accuracy&Params&Pass@1&Params&Accuracy&Params&Accuracy&Params\\
\midrule
\multirow{2}*{No Tuning}&Zero-Shot & 04.78 & N/A & 13.41 & N/A & 22.03 & N/A & 13.41 & N/A\\
&Few-Shot & 55.95 & N/A & 17.68 & N/A & 81.36 & N/A & 51.66 & N/A\\
\midrule\multirow{3}*{\makecell[l]{Prefix\\Injection}}&Prefix Tuning & 02.65 & 0.034B & 00.00 & 0.034B & 28.47 & 0.004B & 10.37 & 0.024B\\
&Prompt Tuning & 04.70 & 0.000B & 26.22 & 0.000B & 23.73 & 0.000B & 18.22 & 0.000B\\
&P-Tuning & 34.19 & 0.001B & 27.44 & 0.001B & 43.05 & 0.001B & 34.89 & 0.001B\\
\midrule\multirow{4}*{\makecell[l]{Weight\\Modulation}}&(IA)$^\text3$ & 08.57 & 0.001B & 31.10 & 0.001B & 23.39 & 0.001B & 21.02 & 0.001B\\
&LoRA & 59.21 & 0.168B & 26.83 & 0.084B & 83.05 & 0.084B & 56.36 & 0.112B\\
&DoRA & 55.34 & 0.043B & 31.10 & 0.169B & 83.39 & 0.169B & 56.61 & 0.127B\\
&rsLoRA & 62.47 & 0.042B & 28.66 & 0.021B & 82.71 & 0.021B & 57.95 & 0.028B\\
\midrule\multirow{4}*{\makecell[l]{Mixture}}&VB-LoRA & 34.27 & 0.677B & 29.27 & 0.673B & 23.73 & 0.674B & 29.09 & 0.675B\\
&MixLoRA & 61.87 & 0.068B & 28.05 & 0.116B & 82.37 & 0.119B & 57.43 & 0.101B\\
&HydraLoRA & 62.47 & 0.092B & 20.12 & 0.079B & 82.71 & 0.082B & 55.10 & 0.084B\\
&\Ours{} (Ours)\cellcolor{C1} & \textbf{65.66}\cellcolor{C1} & 0.106B\cellcolor{C1} & \textbf{32.93}\cellcolor{C1} & 0.090B\cellcolor{C1} & \textbf{83.73}\cellcolor{C1} & 0.016B\cellcolor{C1} & \textbf{60.77}\cellcolor{C1} & 0.070B\cellcolor{C1}\\
\bottomrule
\end{tabular}
}
\vspace{-2.5em}
\end{center}
\end{table*}

\subsection{Experimental Setup}

\textbf{Baselines.}
%
We comprehensively compare our proposed \Ours{} against various types of baseline methods. 
(i) No tuning methods: testing the base LLM directly under zero-shot and few-shot prompting.
(ii) Prefix injection methods: Prefix Tuning \citep{li2021prefix}, Prompt Tuning \citep{lester2021power}, and P-Tuning \citep{liu2021gpt}.
(iii) Weight modulation methods: (IA)$^\text3$ \citep{liu2022few}, LoRA \citep{hu2021lora}, DoRA \citep{liu2024dora}, and rsLoRA \citep{kalajdzievski2023rank}.
(iv) Mixture methods: VB-LoRA \citep{li2024vb}, MixLoRA, \citep{mixlora}, and HydraLoRA~\citep{hydralora}.
For each baseline method, we perform a hyperparameter search and report the best results.

\textbf{Datasets \& evaluation metrics.}
We finetune the base LLM and evaluate them on a diverse set of benchmarks, including GSM8K~\citep{gsm8k} to evaluate mathematical reasoning capabilities, HumanEval~\citep{humaneval} to evaluate code generation capabilities, and ARC-c~\citep{clark2018think} to evaluate knowledge recall capabilities.
For HumanEval, since HumanEval does not contain a training set, we follow ~\cite{hydralora} to finetune the base LLM on CodeAlpaca~\citep{codealpaca} and report the Pass@1 metric on HumanEval.
For all other datasets, we finetune the base LLM on their training split and report the accuracy metric on their test split. In this work, we use Llama 3 8B~\citep{dubey2024llama} as the base LLM. Besides that, we also report the number of activated parameters (in billion, B) under the best-performing hyperparameters.


\textbf{Implementation details.}
We train all methods using the same number of epochs, learning rate schedule, gradient accumulation steps and machine type. 
All methods are trained using the LLaMA-Factory~\citep{llamafactory} framework and evaluated using the OpenCompass~\citep{opencompass} framework. 
For the no-tuning few-shot method, we use 4 shots for GSM8K and HumanEval and 5 shots for ARC-c.



\subsection{Main Results}

We evaluate the performance of various fine-tuning strategies on three representative tasks: HumanEval (code generation), GSM8K (math reasoning), and ARC-c (knowledge recall). As shown in Table~\ref{tab:exp-main}, our \Ours{} consistently outperforms all baselines across these benchmarks while maintaining strong parameter efficiency.

From a performance standpoint, \Ours{} surpasses all baseline methods, achieving an average accuracy improvement of 2.82 over the strongest competing approach. 
Specifically, \Ours{} outperforms the best Prefix Injection baseline by a substantial 25.88, the best Weight Modulation baseline by 2.82, and the strongest Mixture competitor by 3.34 on average across the three tasks. 
On HumanEval, \Ours{} achieves a Pass@1 of 32.93, outperforming the best baseline, (IA)$^3$, by 1.83. 
For GSM8K, \Ours{} attains an accuracy of 65.66, showing a clear gain of 3.19 over the best competitors (rsLoRA and HydraLoRA). 
On ARC-c, \Ours{} reaches 83.73, exceeding the best-performing low-rank method DoRA by 0.34. 
These results highlight the consistent advantages of our reinforcement-trained router across diverse task types. 
Notably, within the Mixture methods, ReMix provides consistent improvements, suggesting that reinforcement-guided, balance-aware routing enhances both reasoning-intensive tasks (e.g., GSM8K) and generation tasks (e.g., HumanEval), while preserving strong retrieval performance on ARC-c.

In terms of parameter efficiency, \Ours{} achieves these performance gains with a competitive budget of only 0.070B trainable parameters. 
Compared to other mixture methods, this represents a 90\% reduction relative to the most parameter-heavy baseline VB-LoRA (0.675B), and a 31\% reduction compared to the most effective baseline MixLoRA (0.101B). 
Even when compared to the lightweight rsLoRA (0.028B), \Ours{} delivers a +2.82 average-accuracy improvement at the cost of only 0.042B more parameters, demonstrating a superior accuracy-to-parameter trade-off. 
Overall, these results confirm that reinforcement-guided mixture routing achieves state-of-the-art accuracy with minimal and often reduced parameter overhead. 




\begin{wrapfigure}{r}{0.21\linewidth}
\centering
\vspace{-5em}
\includegraphics[width=\linewidth]{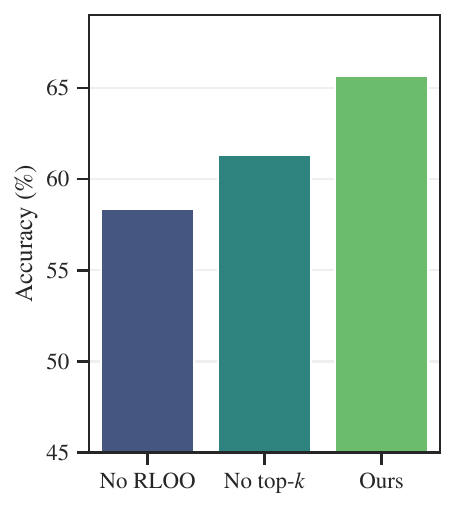}
\vspace{-2em}
\caption{Both of our proposed RLOO and top-$k$ selection contribute significantly to the strong performance of our \Ours{}.}
\label{fig:abla}
\vspace{-2em}
\end{wrapfigure}

\subsection{Ablation Studies}

To understand the contributions of the key components in our proposed \Ours{} (i.e., RLOO for router training and top-$k$ LoRA selection for inference), we conducte ablation studies on GSM8K comparing its performance against the ablated variants with each component removed. The results are presented in Figure~\ref{fig:abla}, which visualizes the accuracy achieved by different configurations.

From Figure~\ref{fig:abla}, we observe that our full \Ours{} method achieves the highest accuracy among all ablated variants. When removing the RLOO from our finetuning procedure \Ours{} (No RLOO), we observe a significant drop in accuracy compared to the full \Ours{}, indicating that RLOO plays a crucial role in enhancing the model's performance. Similarly, disabling the top-$k$ LoRA selection (No top-$k$) also results in lower accuracy than the complete \Ours{}, demonstrating the importance of this component in optimizing the model performance. 
These findings underscore the value of integrating both RLOO and top-$k$ selection into our \Ours{} method.

\begin{table}[t]
\definecolor{C1}{rgb}{0.84, 0.93, 0.98}%
\centering
\caption{The fact that our method significantly outperforms rank-$kr$ LoRA clearly shows that our \Ours{} can activate \textbf{diverse} LoRA subsets: If the activated subset were always the same subset, then the results would be the same as rank-$kr$ LoRA.}\label{tab:kr-vs-k-r}
\vspace{-1em}
\resizebox{0.5\linewidth}{!}{
\begin{tabular}{l|ccc}
\toprule
\textbf{Method}&$k=1$&$k=2$&$k=4$\\
\midrule
Rank-$kr$ LoRA & 56.10&54.51&59.21 \\
\cellcolor{C1}Ours: $k$ rank-$r$ LoRAs &\cellcolor{C1}\textbf{56.18}&\cellcolor{C1}\textbf{59.67}&\cellcolor{C1}\textbf{64.22}\\
\bottomrule
\end{tabular}
}
\vspace{-1em}
\end{table}

\begin{table}[t]
\centering
\caption{Our \Ours{} is not sensitive to routing weight $\omega$.}\label{tab:exp-omega}
\begin{tabular}{l|cc}
\toprule
\textbf{Method}&LoRA-type $\omega$&rsLoRA-type $\omega$\\
\midrule
\Ours{} (Ours) &53.30&55.72\\
\bottomrule
\end{tabular}
\end{table}

\subsection{Diversity of Activated LoRA Subsets}

In this subsection, we empirically verify the diversity of activated LoRA subsets. If the activated subset were always the same subset, then this method would be the same as rank-$kr$ LoRA. Therefore, we compare our method (a mixture of $k$ rank-$r$ LoRAs) with a single rank-$kr$ LoRA, which has the same number of LoRA parameters. The results for $r=8$ on GSM8K under various $k$ are presented in Table~\ref{tab:kr-vs-k-r}.

The fact that our \Ours{} significantly outperforms rank-$kr$ LoRA demonstrates the diversity of selected LoRA subsets. For instance, our accuracy 64.22 for $k=4$ significantly outperforms the accuracy 59.21 of rank-$32$ LoRA. This clearly demonstrates that our method is able to choose different subsets appropriately.

\subsection{Training Efficiency}

In this subsection, we study the training efficiency of our proposed method. Note that MixLoRA can be regarded as an ablated variant where our reinforcement router is replaced with an ordinary learnable router. Hence, we compare our \Ours{} and MixLoRA under comparable training time to show the training efficiency of our proposed \Ours{}. The results are presented in Table~\ref{tab:eff}.

\begin{table}[t]
\definecolor{C1}{rgb}{0.84, 0.93, 0.98}%
\centering
\caption{Our \Ours{} significantly outperforms MixLoRA under similar training time.}\label{tab:eff}
\vspace{-0.5em}
\resizebox{0.6\linewidth}{!}{\begin{tabular}{l|ccc}
\toprule
\textbf{Method}&\textbf{Per-Step Time}&\textbf{Total Time}&\textbf{Accuracy}\\
\midrule
MixLoRA  & 8.95\,s & 1:12:56 & 50.34 \\
\cellcolor{C1}\Ours{} (Ours) &\cellcolor{C1}9.87\,s &\cellcolor{C1}1:28:21&\cellcolor{C1}\textbf{58.38} \\
\bottomrule
\end{tabular}}
\vspace{-1em}
\end{table}

As shown in Table~\ref{tab:eff}, our \Ours{} achieves an accuracy of 58.38\% with a total training time of 1:28:21, while MixLoRA achieves an accuracy of 50.34\% in 1:12:56. Although our \Ours{} consumes only 10\% more training time than MixLoRA, it yields a substantial relative improvement of 15.97\% in accuracy. This demonstrates that our \Ours{} still retains strong performance even under small training compute budget. 



\begin{table}[t]
\centering
\caption{Our \Ours{} consistently improves as we scale up the number $k$ of activated LoRAs.}\label{tab:exp-k}
\vspace{-0.5em}
\resizebox{0.5\linewidth}{!}{
\begin{tabular}{l|cccc}
\toprule
\textbf{Method}&$k=1$&$k=2$&$k=3$&$k=4$\\
\midrule
\Ours{} (Ours) &56.18&59.67&61.33&64.22\\
\bottomrule
\end{tabular}
}
\vspace{-1em}
\end{table}

\subsection{Scaling the Number of Activated LoRAs}

In this subsection, we study how scaling the number $k$ of activated LoRAs benefits the predictive accuracy. Intuitively, the choice of $k$ depends on the tradeoff between efficiency and accuracy. Theoretically, since the number of size-$k$ subsets (i.e., $\binom nk$) increases with $k$ when $k\le n/2$, we would expect accuracy to increase with $k$ correspondingly. To verify this, we conduct experiments with $n=8$ under various $k$ on GSM8K. The results are shown in Table~\ref{tab:exp-k}. Indeed, as shown in the table, Our \Ours{} consistently achieves stronger results under larger $k$ whenever $k\le n/2$.

\begin{wrapfigure}{r}{0.27\linewidth}
\centering
\vspace{-2em}
\includegraphics[width=\linewidth]{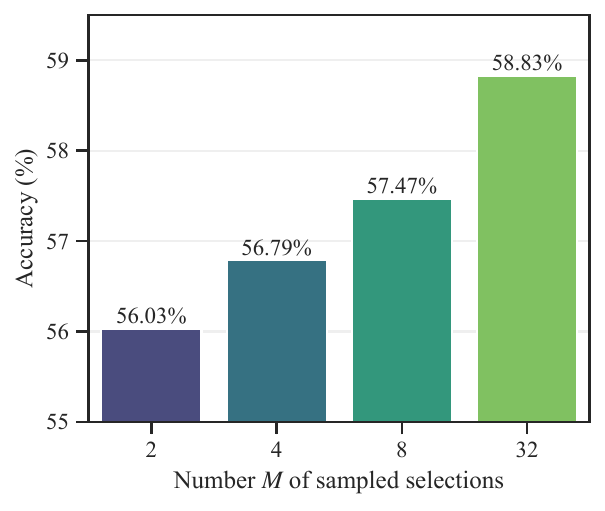}
\vspace{-1em}
\caption{Our proposed \Ours{} can further benefit from scaling up the training compute. In contrast, HydraLoRA and MixLoRA have \textbf{fixed training compute} and thus cannot benefit from scaling up training compute.}
\label{fig:scal-acc}
\vspace{-3em}
\end{wrapfigure}

\subsection{Scaling the Training Compute}

Since our \Ours{} incorporates RL-based gradient estimator, we can effectively scale up training compute by increasing the number $M$ of sampled selections. To evaluate how training compute scaling benefits our \Ours{}, we examine its performance under varying numbers $M$ of sampled selections. As shown in Figure~\ref{fig:scal-acc}, increasing $M$ from 2 to 32 leads to a steady improvement in accuracy, rising from 56.03\% to 58.83\%. This indicates that our \Ours{} effectively leverages additional computational resources to enhance its performance. Notably, the consistent gains observed across different scales suggest that further increases in $M$ could yield even better results. This demonstrates that \Ours{} offers a favorable trade-off between training efficiency and performance. In stark contrast, existing methods do not benefit from training compute scaling because their deterministic training has fixed training compute, which underscores the unique advantage offered by \Ours{} in utilizing increased training compute to achieve improved outcomes.

\subsection{LoRA v.s.\ rsLoRA Routing Weight}\label{ssec:exp-omega}

We compare LoRA-type $\omega:=2/kr$ and rsLoRA-type $\omega:=2/\sqrt{kr}$ under $k=3$ on GSM8K. The results are shown in Table~\ref{tab:exp-omega}. We find that our method \Ours{} is not sensitive to the choice of $\omega$. As shown in Table~\ref{tab:exp-omega}, the performance has only very small difference under these two versions of $\omega$.

\section{Related Work}

Parameter-efficient fine-tuning (PEFT) aims to reduce the number of trainable parameters while achieving strong task performance. Due to the page limit, please refer to Appendix~\ref{app:related} for related work on general PEFT.
More recent efforts in PEFT have explored new multi-LoRA architectures that go beyond single low-rank adapters by explicitly restructuring how multiple LoRA modules are organized and combined, offering advantages on complex data distributions. LoraHub \citep{huang2023lorahub} introduces a dynamic composition framework that integrates multiple LoRAs at the architectural level, enabling cross-task generalization without retraining by assembling adapters into a unified pipeline. MultiLoRA \citep{wang2023multilora} modifies the structural initialization of LoRA subspaces and horizontally expands adapters across layers, thereby mitigating the dominance of top singular vectors and achieving more balanced representations in multi-task learning. HydraLoRA \citep{hydralora} departs from the symmetric LoRA design and proposes an asymmetric architecture that decouples the projection and update pathways, substantially improving parameter and training efficiency. Beyond linear compositions, S’MoRE \citep{smore} integrates LoRA with mixture-of-experts style routing by hierarchically decomposing expert weights into low-rank residual components and routing them through a structured multi-layer architecture. Meanwhile, LoRA-Flow \citep{wang2024loraflow} rethinks the architecture for generative tasks by embedding a lightweight, token-level fusion gate that dynamically modulates multiple LoRAs during inference, and MultLFG \citep{roy2025multlfg} introduces a frequency-aware fusion mechanism that structurally guides LoRA composition across denoising steps.

\section{Conclusion}
In this paper, we investigate the problem of imbalanced routing weights that hinder effective LoRA utilization, and propose a reinforcement-based router design named \Ours{} to address this problem.
Extensive experiments across diverse benchmarks demonstrate that our \Ours{} consistently outperforms state-of-the-art parameter-efficient finetuning methods, achieving superior predictive power and computational efficiency.





\bibliography{output}

@article{roy2025multlfg,
  title={MultLFG: Training-free Multi-LoRA composition using Frequency-domain Guidance},
  author={Roy, Aniket and Suin, Maitreya and Shah, Ketul and Chellappa, Rama},
  journal={arXiv preprint arXiv:2505.20525},
  year={2025}
}

@article{wang2024loraflow,
  title={Lora-flow: Dynamic lora fusion for large language models in generative tasks},
  author={Wang, Hanqing and Ping, Bowen and Wang, Shuo and Han, Xu and Chen, Yun and Liu, Zhiyuan and Sun, Maosong},
  journal={arXiv preprint arXiv:2402.11455},
  year={2024}
}

@article{wang2023multilora,
  title={Multilora: Democratizing lora for better multi-task learning},
  author={Wang, Yiming and Lin, Yu and Zeng, Xiaodong and Zhang, Guannan},
  journal={arXiv preprint arXiv:2311.11501},
  year={2023}
}

@article{huang2023lorahub,
  title={Lorahub: Efficient cross-task generalization via dynamic lora composition},
  author={Huang, Chengsong and Liu, Qian and Lin, Bill Yuchen and Pang, Tianyu and Du, Chao and Lin, Min},
  journal={arXiv preprint arXiv:2307.13269},
  year={2023}
}

@article{yang2023bayesian,
  title={Bayesian low-rank adaptation for large language models},
  author={Yang, Adam X and Robeyns, Maxime and Wang, Xi and Aitchison, Laurence},
  journal={arXiv preprint arXiv:2308.13111},
  year={2023}
}

@article{zhang2023adalora,
  title={Adalora: Adaptive budget allocation for parameter-efficient fine-tuning},
  author={Zhang, Qingru and Chen, Minshuo and Bukharin, Alexander and Karampatziakis, Nikos and He, Pengcheng and Cheng, Yu and Chen, Weizhu and Zhao, Tuo},
  journal={arXiv preprint arXiv:2303.10512},
  year={2023}
}

@article{valipour2022dylora,
  title={Dylora: Parameter efficient tuning of pre-trained models using dynamic search-free low-rank adaptation},
  author={Valipour, Mojtaba and Rezagholizadeh, Mehdi and Kobyzev, Ivan and Ghodsi, Ali},
  journal={arXiv preprint arXiv:2210.07558},
  year={2022}
}

@article{hu2022lora,
  title={Lora: Low-rank adaptation of large language models.},
  author={Hu, Edward J and Shen, Yelong and Wallis, Phillip and Allen-Zhu, Zeyuan and Li, Yuanzhi and Wang, Shean and Wang, Lu and Chen, Weizhu and others},
  journal={ICLR},
  volume={1},
  number={2},
  pages={3},
  year={2022}
}

@article{ruckle2020adapterdrop,
  title={Adapterdrop: On the efficiency of adapters in transformers},
  author={R{\"u}ckl{\'e}, Andreas and Geigle, Gregor and Glockner, Max and Beck, Tilman and Pfeiffer, Jonas and Reimers, Nils and Gurevych, Iryna},
  journal={arXiv preprint arXiv:2010.11918},
  year={2020}
}

@article{he2022sparseadapter,
  title={Sparseadapter: An easy approach for improving the parameter-efficiency of adapters},
  author={He, Shwai and Ding, Liang and Dong, Daize and Zhang, Miao and Tao, Dacheng},
  journal={arXiv preprint arXiv:2210.04284},
  year={2022}
}

@inproceedings{jie2023revisiting,
  title={Revisiting the parameter efficiency of adapters from the perspective of precision redundancy},
  author={Jie, Shibo and Wang, Haoqing and Deng, Zhi-Hong},
  booktitle={Proceedings of the IEEE/CVf international conference on computer vision},
  pages={17217--17226},
  year={2023}
}

@article{li2021prefix,
  title={Prefix-tuning: Optimizing continuous prompts for generation},
  author={Li, Xiang Lisa and Liang, Percy},
  journal={arXiv preprint arXiv:2101.00190},
  year={2021}
}

@article{le2024revisiting,
  title={Revisiting prefix-tuning: Statistical benefits of reparameterization among prompts},
  author={Le, Minh and Nguyen, Chau and Nguyen, Huy and Tran, Quyen and Le, Trung and Ho, Nhat},
  journal={arXiv preprint arXiv:2410.02200},
  year={2024}
}

@article{chen2022inducer,
  title={Inducer-tuning: Connecting Prefix-tuning and Adapter-tuning},
  author={Chen, Yifan and Hazarika, Devamanyu and Namazifar, Mahdi and Liu, Yang and Jin, Di and Hakkani-Tur, Dilek},
  journal={arXiv preprint arXiv:2210.14469},
  year={2022}
}

@article{petrov2023prompting,
  title={When do prompting and prefix-tuning work? a theory of capabilities and limitations},
  author={Petrov, Aleksandar and Torr, Philip HS and Bibi, Adel},
  journal={arXiv preprint arXiv:2310.19698},
  year={2023}
}

@article{liu2021p,
  title={P-tuning v2: Prompt tuning can be comparable to fine-tuning universally across scales and tasks},
  author={Liu, Xiao and Ji, Kaixuan and Fu, Yicheng and Tam, Weng Lam and Du, Zhengxiao and Yang, Zhilin and Tang, Jie},
  journal={arXiv preprint arXiv:2110.07602},
  year={2021}
}

@article{shi2023dept,
  title={Dept: Decomposed prompt tuning for parameter-efficient fine-tuning},
  author={Shi, Zhengxiang and Lipani, Aldo},
  journal={arXiv preprint arXiv:2309.05173},
  year={2023}
}

@article{lester2021power,
  title={The power of scale for parameter-efficient prompt tuning},
  author={Lester, Brian and Al-Rfou, Rami and Constant, Noah},
  journal={arXiv preprint arXiv:2104.08691},
  year={2021}
}

@article{zang2022unified,
  title={Unified vision and language prompt learning},
  author={Zang, Yuhang and Li, Wei and Zhou, Kaiyang and Huang, Chen and Loy, Chen Change},
  journal={arXiv preprint arXiv:2210.07225},
  year={2022}
}

@inproceedings{wang2022no,
  title={No more fine-tuning? an experimental evaluation of prompt tuning in code intelligence},
  author={Wang, Chaozheng and Yang, Yuanhang and Gao, Cuiyun and Peng, Yun and Zhang, Hongyu and Lyu, Michael R},
  booktitle={Proceedings of the 30th ACM joint European software engineering conference and symposium on the foundations of software engineering},
  pages={382--394},
  year={2022}
}

@inproceedings{kool2019buy,
  title={Buy 4 {REINFORCE} samples, get a baseline for free!},
  author={Kool, Wouter and van Hoof, Herke and Welling, Max},
  booktitle={ICLR 2019 workshop: Deep RL Meets Structured Prediction},
  year={2019}
}

@article{dubey2024llama,
  title={The llama 3 herd of models},
  author={Dubey, Abhimanyu and Jauhri, Abhinav and Pandey, Abhinav and Kadian, Abhishek and Al-Dahle, Ahmad and Letman, Aiesha and Mathur, Akhil and Schelten, Alan and Yang, Amy and Fan, Angela and others},
  journal={arXiv preprint arXiv:2407.21783},
  year={2024}
}

@article{hu2021lora,
  title={Lora: Low-rank adaptation of large language models},
  author={Hu, Edward J and Shen, Yelong and Wallis, Phillip and Allen-Zhu, Zeyuan and Li, Yuanzhi and Wang, Shean and Wang, Lu and Chen, Weizhu},
  journal={arXiv preprint arXiv:2106.09685},
  year={2021}
}

@inproceedings{chen2026influence,
title={Influence-preserving proxies for gradient-based data selection in {LLM} finetuning},
author={Chen, Sirui and Qi, Yunzhe and Ai, Mengting and Sun, Yifan and Qiu, Ruizhong and Zou, Jiaru and He, Jingrui},
booktitle={The Fourteenth International Conference on Learning Representations},
year={2026},
}

@inproceedings{yu2026planetalign,
title={{PlanetAlign:} A comprehensive {Python} library for benchmarking network alignment},
author={Yu, Qi and Zeng, Zhichen and Yan, Yuchen and Liu, Zhining and Jing, Baoyu and Qiu, Ruizhong and Azad, Ariful and Tong, Hanghang},
booktitle={The Fourteenth International Conference on Learning Representations},
year={2026},
}

@inproceedings{li2025graph,
title={Graph data selection for domain adaptation: {A} model-free approach},
author={Li, Ting-Wei and Qiu, Ruizhong and Tong, Hanghang},
booktitle={Advances in Neural Information Processing Systems 38},
year={2025},
}

@inproceedings{zou2025transformer,
title={Transformer copilot: {Learning} from the mistake log in {LLM} fine-tuning},
author={Zou, Jiaru and Ban, Yikun and Li, Zihao and Qi, Yunzhe and Qiu, Ruizhong and Yang, Ling and He, Jingrui},
booktitle={Advances in Neural Information Processing Systems 38},
year={2025},
}

@inproceedings{xu2024discrete,
title={Discrete-state continuous-time diffusion for graph generation},
author={Xu, Zhe and Qiu, Ruizhong and Chen, Yuzhong and Chen, Huiyuan and Fan, Xiran and Pan, Menghai and Zeng, Zhichen and Das, Mahashweta and Tong, Hanghang},
booktitle={Advances in Neural Information Processing Systems 37},
year={2024},
}

@inproceedings{lin2024backtime,
title={{BackTime:} Backdoor attacks on multivariate time series forecasting},
author={Lin, Xiao and Liu, Zhining and Fu, Dongqi and Qiu, Ruizhong and Tong, Hanghang},
booktitle={Advances in Neural Information Processing Systems 37},
year={2024},
}

@inproceedings{liu2025breaking,
title={Breaking silos: {Adaptive} model fusion unlocks better time series forecasting},
author={Liu, Zhining and Yang, Ze and Lin, Xiao and Qiu, Ruizhong and Wei, Tianxin and Zhu, Yada and Hamann, Hendrik and He, Jingrui and Tong, Hanghang},
booktitle={Proceedings of the 42nd International Conference on Machine Learning},
year={2025},
}

@inproceedings{zeng2024graph,
title={Graph mixup on approximate {Gromov–Wasserstein} geodesics},
author={Zeng, Zhichen and Qiu, Ruizhong and Xu, Zhe and Liu, Zhining and Yan, Yuchen and Wei, Tianxin and Ying, Lei and He, Jingrui and Tong, Hanghang},
booktitle={Proceedings of the 41st International Conference on Machine Learning},
year={2024},
}

@inproceedings{liu2024class,
title={Class-imbalanced graph learning without class rebalancing},
author={Liu, Zhining and Qiu, Ruizhong and Zeng, Zhichen and Yoo, Hyunsik and Zhou, David and Xu, Zhe and Zhu, Yada and Weldemariam, Kommy and He, Jingrui and Tong, Hanghang},
booktitle={Proceedings of the 41st International Conference on Machine Learning},
year={2024},
}

@inproceedings{bao2025latte,
title={Latte: Collaborative test-time adaptation of vision-language models in federated learning},
author={Bao, Wenxuan and Deng, Ruxi and Qiu, Ruizhong and Wei, Tianxin and Tong, Hanghang and He, Jingrui},
booktitle={Proceedings of the IEEE/CVF International Conference on Computer Vision},
year={2025},
}

@article{zeng2026harnessing,
title={Harnessing consistency for robust test-time {LLM} ensemble},
author={Zeng, Zhichen and Yu, Qi and Lin, Xiao and Qiu, Ruizhong and Ning, Xuying and Wei, Tianxin and Yan, Yuchen and He, Jingrui and Tong, Hanghang},
journal={Findings of EACL 2026},
year={2026},
}

@inproceedings{he2026powergrow,
title={{PowerGrow:} Feasible co-growth of structures and dynamics for power grid synthesis},
author={He, Xinyu and Xiao, Chenhan and Li, Haoran and Qiu, Ruizhong and Xu, Zhe and Weng, Yang and He, Jingrui and Tong, Hanghang},
booktitle={Proceedings of the 32nd ACM SIGKDD Conference on Knowledge Discovery and Data Mining},
year={2026},
}

@inproceedings{yoo2025generalizable,
title={Generalizable recommender system during temporal popularity distribution shifts},
author={Yoo, Hyunsik and Qiu, Ruizhong and Xu, Charlie and Wang, Fei and Tong, Hanghang},
booktitle={Proceedings of the 31st ACM SIGKDD Conference on Knowledge Discovery and Data Mining},
year={2025},
}

@inproceedings{liu2024aim,
title={{AIM:} Attributing, interpreting, mitigating data-encoded unfairness},
author={Liu, Zhining and Qiu, Ruizhong and Zeng, Zhichen and Zhu, Yada and Hamann, Hendrik and Tong, Hanghang},
booktitle={Proceedings of the 30th ACM SIGKDD Conference on Knowledge Discovery and Data Mining},
year={2024},
}

@inproceedings{wang2023networked,
title={Networked time series imputation via position-aware graph enhanced variational autoencoders},
author={Wang, Dingsu and Yan, Yuchen and Qiu, Ruizhong and Zhu, Yada and Guan, Kaiyu and Margenot, Andrew J. and Tong, Hanghang},
booktitle={Proceedings of the 29th ACM SIGKDD Conference on Knowledge Discovery and Data Mining},
year={2023},
}

@inproceedings{yoo2025embracing,
title={Embracing plasticity: {Balancing} stability and plasticity in continual recommender systems},
author={Yoo, Hyunsik and Kang, SeongKu and Qiu, Ruizhong and Xu, Charlie and Wang, Fei and Tong, Hanghang},
booktitle={Proceedings of the 48th International ACM SIGIR Conference on Research and Development in Information Retrieval},
year={2025},
}

@inproceedings{zhang2026guiding,
title={Guiding generative recommender systems with structured human priors via multi-head decoding},
author={Zhang, Yunkai and Zhang, Qiang and Yang, Diji and Lin, Ryan and Qiu, Ruizhong and Zhang, Benyu and Yu, Hanchao and Liu, Jason and Xia, Yinglong and Zhao, Zhuokai and Zhang, Lizhu and Fan, Xiangjun and Yu, Zhuoran and Kumar, Abhishek and Zheng, Zeyu},
booktitle={Proceedings of the ACM Web Conference 2026},
year={2026},
}

@inproceedings{lin2026mixture,
title={Mixture of sequence: {Theme-aware} mixture-of-experts for long-sequence recommendation},
author={Lin, Xiao and Tang, Zhicheng and Cong, Weilin and Hang, Mengyue and Wang, Kai and Wang, Yajuan and Zeng, Zhichen and Li, Ting-Wei and Yoo, Hyunsik and Liu, Zhining and Ning, Xuying and Qiu, Ruizhong and Chen, Wen-Yen and Chang, Shuo and Jin, Rong and Li, Huayu and Tong, Hanghang},
booktitle={Proceedings of the ACM Web Conference 2026},
year={2026},
}

@inproceedings{yoo2024ensuring,
title={Ensuring user-side fairness in dynamic recommender systems},
author={Yoo, Hyunsik and Zeng, Zhichen and Kang, Jian and Qiu, Ruizhong and Zhou, David and Liu, Zhining and Wang, Fei and Xu, Charlie and Chan, Eunice and Tong, Hanghang},
booktitle={Proceedings of the ACM Web Conference 2024},
year={2024},
}

@inproceedings{he2024on,
title={On the sensitivity of individual fairness: {Measures} and robust algorithms},
author={He, Xinyu and Kang, Jian and Qiu, Ruizhong and Wang, Fei and Sepulveda, Jose and Tong, Hanghang},
booktitle={Proceedings of the 33rd ACM International Conference on Information and Knowledge Management},
year={2024},
}

@article{zeng2026pave,
title={Pave your own path: {Graph} gradual domain adaptation on fused {Gromov–Wasserstein} geodesics},
author={Zeng, Zhichen and Qiu, Ruizhong and Bao, Wenxuan and Wei, Tianxin and Lin, Xiao and Yan, Yuchen and Abdelzaher, Tarek F. and Han, Jiawei and Tong, Hanghang},
journal={Transactions on Machine Learning Research},
year={2026},
}

@article{wei2026diffkgw,
title={{DiffKGW:} Stealthy and robust diffusion model watermarking},
author={Wei, Tianxin and Qiu, Ruizhong and Chen, Yifan and Qi, Yunzhe and Lin, Jiacheng and Bao, Wenxuan and Xu, Wenju and Nag, Sreyashi and Li, Ruirui and Lu, Hanqing and Wang, Zhengyang and Luo, Chen and Liu, Hui and Wang, Suhang and He, Jingrui and He, Qi and Tang, Xianfeng},
journal={Transactions on Machine Learning Research},
year={2026},
}

@article{li2026beyond,
title={Beyond log likelihood: {Probability-based} objectives for supervised fine-tuning across the model capability continuum},
author={Li, Gaotang and Qiu, Ruizhong and Chen, Xiusi and Ji, Heng and Tong, Hanghang},
journal={ICLR 2026 Workshop on Scaling Post-training for LLMs},
year={2026},
}

@article{bartan2025fineamp,
title={{FineAMP:} Optimization-based automatic mixed precision quantization for efficient diffusion model inference},
author={Bartan, Burak and Qiu, Ruizhong and Esteves, Rafael and Ren, Yuwei and Zeng, Weiliang Will and Chen, An},
journal={The 17th International OPT Workshop on Optimization for Machine Learning},
year={2025},
}

@article{li2025haystack,
title={Haystack engineering: {Context} engineering meets the long-context challenge in large language models},
author={Li, Mufei and Fu, Dongqi and Wang, Limei and Zhang, Si and Zeng, Hanqing and Sancak, Kaan and Qiu, Ruizhong and Wang, Haoyu Peter and He, Xiaoxin and Bresson, Xavier and Xia, Yinglong and Sun, Chonglin and Li, Pan},
journal={NeurIPS 2025 Workshop on Evaluating the Evolving LLM Lifecycle: Benchmarks, Emergent Abilities, and Scaling},
year={2025},
}

@article{wei2026agentic,
title={Agentic reasoning for large language models: {A} survey},
author={Wei, Tianxin and Li, Ting-Wei and Liu, Zhining and Ning, Xuying and Yang, Ze and Zou, Jiaru and Zeng, Zhichen and Qiu, Ruizhong and Lin, Xiao and Fu, Dongqi and Li, Zihao and Ai, Mengting and Zhou, Duo and Bao, Wenxuan and Li, Yunzhe and Li, Gaotang and Qian, Cheng and Wang, Yu and Tang, Xiangru and Xiao, Yin and Fang, Liri and Liu, Hui and Tang, Xianfeng and Zhang, Yuji and Wang, Chi and You, Jiaxuan and Ji, Heng and Tong, Hanghang and He, Jingrui},
journal={arXiv preprint},
year={2026},
}

@article{zeng2026subspace,
title={Subspace alignment for vision-language model test-time adaptation},
author={Zeng, Zhichen and Bao, Wenxuan and Lin, Xiao and Qiu, Ruizhong and Wei, Tianxin and Ning, Xuying and Yan, Yuchen and Luo, Chen and Cheng, Monica Xiao and He, Jingrui and Tong, Hanghang},
journal={arXiv preprint},
year={2026},
}

@article{cui2026adafuse,
title={{AdaFuse:} Adaptive ensemble decoding with test-time scaling for {LLMs}},
author={Cui, Chengming and Wei, Tianxin and Chen, Ziyi and Qiu, Ruizhong and Zeng, Zhichen and Liu, Zhining and Ning, Xuying and Zhou, Duo and He, Jingrui},
journal={arXiv preprint},
year={2026},
}

@article{wei2025cofirec,
title={{CoFiRec:} Coarse-to-fine tokenization for generative recommendation},
author={Wei, Tianxin and Ning, Xuying and Chen, Xuxing and Qiu, Ruizhong and Hou, Yupeng and Xie, Yan and Yang, Shuang and Hua, Zhigang and He, Jingrui},
journal={arXiv preprint},
year={2025},
}

@article{zou2025latent,
title={Latent collaboration in multi-agent systems},
author={Zou, Jiaru and Yang, Xiyuan and Qiu, Ruizhong and Li, Gaotang and Tieu, Katherine and Lu, Pan and Shen, Ke and Tong, Hanghang and Choi, Yejin and He, Jingrui and Zou, James and Wang, Mengdi and Yang, Ling},
journal={arXiv preprint},
year={2025},
}

@article{zeng2025hierarchical,
title={Hierarchical {LoRA} {MoE} for efficient {CTR} model scaling},
author={Zeng, Zhichen and Hang, Mengyue and Liu, Xiaolong and Liu, Xiaoyi and Lin, Xiao and Qiu, Ruizhong and Wei, Tianxin and Liu, Zhining and Yuan, Siyang and Yang, Chaofei and Liu, Yiqun and Yin, Hang and Yang, Jiyan and Tong, Hanghang},
journal={arXiv preprint},
year={2025},
}

@inproceedings{hydralora,
  title={HydraLoRA: An Asymmetric LoRA Architecture for Efficient Fine-Tuning},
  author={Tian, Chunlin and Shi, Zhan and Guo, Zhijiang and Li, Li and Xu, Chengzhong},
  booktitle={Advances in Neural Information Processing Systems (NeurIPS)},
  year={2024}
}

@article{mixlora,
  title={Mixlora: Enhancing large language models fine-tuning with lora based mixture of experts},
  author={Li, Dengchun and Ma, Yingzi and Wang, Naizheng and Cheng, Zhiyuan and Duan, Lei and Zuo, Jie and Yang, Cal and Tang, Mingjie},
  journal={arXiv preprint arXiv:2404.15159},
  year={2024}
}

@article{clark2018think,
  title={Think you have solved question answering? try arc, the ai2 reasoning challenge},
  author={Clark, Peter and Cowhey, Isaac and Etzioni, Oren and Khot, Tushar and Sabharwal, Ashish and Schoenick, Carissa and Tafjord, Oyvind},
  journal={arXiv preprint arXiv:1803.05457},
  year={2018}
}

@inproceedings{llamafactory,
  title={LlamaFactory: Unified Efficient Fine-Tuning of 100+ Language Models},
  author={Yaowei Zheng and Richong Zhang and Junhao Zhang and Yanhan Ye and Zheyan Luo and Zhangchi Feng and Yongqiang Ma},
  booktitle={Proceedings of the 62nd Annual Meeting of the Association for Computational Linguistics (Volume 3: System Demonstrations)},
  address={Bangkok, Thailand},
  publisher={Association for Computational Linguistics},
  year={2024},
  url={http://arxiv.org/abs/2403.13372}
}

@misc{opencompass,
    title={OpenCompass: A Universal Evaluation Platform for Foundation Models},
    author={OpenCompass Contributors},
    howpublished = {\url{https://github.com/open-compass/opencompass}},
    year={2023}
}

@article{gsm8k,
  author       = {Karl Cobbe and
                  Vineet Kosaraju and
                  Mohammad Bavarian and
                  Mark Chen and
                  Heewoo Jun and
                  Lukasz Kaiser and
                  Matthias Plappert and
                  Jerry Tworek and
                  Jacob Hilton and
                  Reiichiro Nakano and
                  Christopher Hesse and
                  John Schulman},
  title        = {Training Verifiers to Solve Math Word Problems},
  journal      = {CoRR},
  volume       = {abs/2110.14168},
  year         = {2021},
  url          = {https://arxiv.org/abs/2110.14168},
  eprinttype    = {arXiv},
  eprint       = {2110.14168},
  bibsource    = {dblp computer science bibliography, https://dblp.org}
}

@article{humaneval,
  author       = {Mark Chen and
                  Jerry Tworek and
                  Heewoo Jun and
                  Qiming Yuan and
                  Henrique Pond{\'{e}} de Oliveira Pinto and
                  Jared Kaplan and
                  Harri Edwards and
                  Yuri Burda and
                  Nicholas Joseph and
                  Greg Brockman and
                  Alex Ray and
                  Raul Puri and
                  Gretchen Krueger and
                  Michael Petrov and
                  Heidy Khlaaf and
                  Girish Sastry and
                  Pamela Mishkin and
                  Brooke Chan and
                  Scott Gray and
                  Nick Ryder and
                  Mikhail Pavlov and
                  Alethea Power and
                  Lukasz Kaiser and
                  Mohammad Bavarian and
                  Clemens Winter and
                  Philippe Tillet and
                  Felipe Petroski Such and
                  Dave Cummings and
                  Matthias Plappert and
                  Fotios Chantzis and
                  Elizabeth Barnes and
                  Ariel Herbert{-}Voss and
                  William Hebgen Guss and
                  Alex Nichol and
                  Alex Paino and
                  Nikolas Tezak and
                  Jie Tang and
                  Igor Babuschkin and
                  Suchir Balaji and
                  Shantanu Jain and
                  William Saunders and
                  Christopher Hesse and
                  Andrew N. Carr and
                  Jan Leike and
                  Joshua Achiam and
                  Vedant Misra and
                  Evan Morikawa and
                  Alec Radford and
                  Matthew Knight and
                  Miles Brundage and
                  Mira Murati and
                  Katie Mayer and
                  Peter Welinder and
                  Bob McGrew and
                  Dario Amodei and
                  Sam McCandlish and
                  Ilya Sutskever and
                  Wojciech Zaremba},
  title        = {Evaluating Large Language Models Trained on Code},
  journal      = {CoRR},
  volume       = {abs/2107.03374},
  year         = {2021},
  url          = {https://arxiv.org/abs/2107.03374},
  eprinttype    = {arXiv},
  eprint       = {2107.03374},
  bibsource    = {dblp computer science bibliography, https://dblp.org}
}

@misc{codealpaca,
  author = {Sahil Chaudhary},
  title = {Code Alpaca: An Instruction-following LLaMA model for code generation},
  year = {2023},
  publisher = {GitHub},
  journal = {GitHub repository},
  howpublished = {\url{https://github.com/sahil280114/codealpaca}},
}

@article{smore,
  title={{S'MoRE}: Structural Mixture of Residual Experts for LLM Fine-tuning},
  author={Zeng, Hanqing and Xia, Yinglong and Zhao, Zhuokai and Jiang, Gilbert and Zhang, Qiang and Liu, Jiayi and Zhang, Lizhu and Fan, Xiangjun and Zhang, Benyu},
  journal={arXiv preprint arXiv:2504.06426},
  year={2025}
}

@inproceedings{he2015delving,
  title={Delving deep into rectifiers: Surpassing human-level performance on imagenet classification},
  author={He, Kaiming and Zhang, Xiangyu and Ren, Shaoqing and Sun, Jian},
  booktitle={Proceedings of the IEEE international conference on computer vision},
  pages={1026--1034},
  year={2015}
}

@article{grendar2006entropy,
  title={Entropy and effective support size},
  author={Grendar, Marian},
  journal={Entropy},
  volume={8},
  number={3},
  pages={169--174},
  year={2006}
}

@article{birnbaum1942inequality,
  title={An inequality for {Mill}'s ratio},
  author={Birnbaum, Zygmunt Wilhelm},
  journal={The Annals of Mathematical Statistics},
  volume={13},
  number={2},
  pages={245--246},
  year={1942},
  publisher={Institute of Mathematical Statistics}
}

@article{gordon1941values,
  title={Values of {Mills}' ratio of area to bounding ordinate and of the normal probability integral for large values of the argument},
  author={Gordon, Robert D.},
  journal={The Annals of Mathematical Statistics},
  volume={12},
  number={3},
  pages={364--366},
  year={1941},
  publisher={JSTOR}
}

@article{liu2021gpt,
  title={GPT Understands, Too},
  author={Liu, Xiao and Zheng, Yanan and Du, Zhengxiao and Ding, Ming and Qian, Yujie and Yang, Zhilin and Tang, Jie},
  journal={arXiv preprint arXiv:2103.10385},
  year={2021}
}

@article{liu2022few,
  title={Few-shot parameter-efficient fine-tuning is better and cheaper than in-context learning},
  author={Liu, Haokun and Tam, Derek and Muqeeth, Mohammed and Mohta, Jay and Huang, Tenghao and Bansal, Mohit and Raffel, Colin A},
  journal={Advances in Neural Information Processing Systems},
  volume={35},
  pages={1950--1965},
  year={2022}
}

@inproceedings{liu2024dora,
  title={{DoRA}: Weight-decomposed low-rank adaptation},
  author={Liu, Shih-Yang and Wang, Chien-Yi and Yin, Hongxu and Molchanov, Pavlo and Wang, Yu-Chiang Frank and Cheng, Kwang-Ting and Chen, Min-Hung},
  booktitle={Forty-first International Conference on Machine Learning},
  year={2024}
}

@article{kalajdzievski2023rank,
  title={A rank stabilization scaling factor for fine-tuning with {LoRA}},
  author={Kalajdzievski, Damjan},
  journal={arXiv preprint arXiv:2312.03732},
  year={2023}
}

@article{li2024vb,
  title={{VB-LoRA}: Extreme parameter efficient fine-tuning with vector banks},
  author={Li, Yang and Han, Shaobo and Ji, Shihao},
  journal={Advances in Neural Information Processing Systems},
  volume={37},
  pages={16724--16751},
  year={2024}
}
\bibliographystyle{iclr2026_conference}

\newpage
\appendix

\addtocontents{toc}{\protect\setcounter{tocdepth}{2}}
\tableofcontents

\section{Related Work (Cont'd)}\label{app:related}
While neural networks has been prevalent in various domains \citep{cui2026adafuse,he2026powergrow,yu2026planetalign,he2024on,yoo2025embracing,yoo2025generalizable,yoo2024ensuring,wang2023networked}, Transformers have nowadays become the \emph{de facto} neural architecture \citep{wei2026agentic,wei2025cofirec,wei2026diffkgw,chen2026influence,li2026beyond,li2025haystack,li2025graph,zeng2026subspace,zeng2026pave,zeng2026harnessing,zeng2025hierarchical,zeng2024graph,zhang2026guiding,lin2026mixture,lin2024backtime,bartan2025fineamp,zou2025transformer,zou2025latent,liu2025breaking,liu2024class,liu2024aim,bao2025latte,xu2024discrete}. PEFT methods for Transformers can be broadly categorized into four groups: prompt tuning, prefix tuning, adapter-based methods, and low-rank adaptation methods. Early methods such as prompt tuning \citep{liu2021p, shi2023dept, lester2021power, zang2022unified, wang2022no} and prefix tuning \citep{li2021prefix, le2024revisiting, chen2022inducer, petrov2023prompting} introduce small continuous prompts, but often struggle to scale to deeper layers or larger models due to limited expressivity. Adapter-based methods \citep{he2022sparseadapter, ruckle2020adapterdrop, jie2023revisiting} mitigate some of these issues by inserting lightweight bottleneck modules into transformer layers. However, as the depth and dimensionality of models increase, the parameter overhead of adapters can become substantial, creating significant bottlenecks in computation and scalability. To address these limitations, low-rank adaptation methods \citep{hu2021lora, valipour2022dylora, zhang2023adalora, yang2023bayesian} are proposed. These methods inject rank-constrained updates into weight matrices, striking a favorable balance between expressivity and parameter cost, and have become a \emph{de facto} standard for many adaptation tasks. Specifically, LoRA \citep{hu2022lora} introduces two trainable low-rank matrices while keeping the original model weights frozen. By training these matrices to approximate parameter perturbations, LoRA achieves effective fine-tuning with minimal overhead. Building on this idea, DyLoRA \citep{valipour2022dylora} dynamically trains LoRA modules across a range of ranks within a predefined budget rather than fixing the rank. AdaLoRA \citep{zhang2023adalora} reformulates parameter perturbations using singular value decomposition (SVD), fine-tuning across the three SVD components for improved flexibility. Laplace-LoRA \citep{yang2023bayesian} takes a Bayesian perspective, applying a post-hoc Laplace approximation to the posterior distribution over LoRA parameters, thereby offering a principled uncertainty-aware extension.

\section{Theoretical Proofs}\label{app:theo}

\subsection{Proof of Theorem~\ref{THM:imb}}\label{app:theo-imb}

Before stating our proof of Theorem~\ref{THM:imb}, we present a few technical lemmata that we will employ.

Let $\varphi(z):=\frac1{\sqrt{2\uppi}}\RM e^{-x^2/2}$, $\varPhi(z):=\int_{-\infty}^z\varphi(x)\DD x$, and $\BAR\varPhi(z):=1-\varPhi(z)$ ($z\in\BB R$) denote the probability density function, the cumulative distribution function, and the complementary cumulative distribution function of the standard Gaussian distribution $\CAL N(0,1)$, respectively.

\begin{LEM}[a Gaussian gap estimate]\label{LEM:0}
For every $z\in\BB R$ and $\alpha>0$,
\AL{\varPhi(z+\alpha)-\varPhi(z)\le\frac{\alpha}{\sqrt{2\uppi}}.}
\end{LEM}


\begin{proof}
Since $\varPhi'(t)=\varphi(t)=\frac{\RM e^{-t^2/2}}{\sqrt{2\uppi}}$, then
\AL{
&\varPhi(z+\alpha)-\varPhi(z)=\int_z^{z+\alpha}\varPhi'(t)\DD t=\int_z^{z+\alpha}\frac{\RM e^{-t^2/2}}{\sqrt{2\uppi}}\DD t\le\int_z^{z+\alpha}\frac1{\sqrt{2\uppi}}\DD t=\frac{\alpha}{\sqrt{2\uppi}}
.\qedhere}
\end{proof}

\begin{LEM}[a Gaussian upper-tail gap estimate]\label{LEM:1}
For any $z\ge0$ and any $\alpha>0$,
\AL{\varPhi(z+\alpha)-\varPhi(z)\le\sqrt{2\uppi}\,(\varPhi(\alpha)-\varPhi(0))\varphi(z)\le \alpha\,\varphi(z).}
\end{LEM}

\begin{proof}
Define a function $h:\BB R_{\ge0}\to\BB R$ as
\AL{h(z):=\frac{\varPhi(z+\alpha)-\varPhi(z)}{\varphi(z)},\qquad z\ge0.}
Since $\varPhi'(z)=\varphi(z)$, and $\varphi'(z)=-z\varphi(z)$, then
\AL{
h'(z)&=\frac{(\varphi(z+\alpha) - \varphi(z))\varPhi'(z)+(\varPhi(z+\alpha) - \varPhi(z))(-\varphi'(z))}{\varphi(z)^2}
\\&=\frac{(\varphi(z+\alpha) - \varphi(z))\varphi(z)+(\varPhi(z+\alpha) - \varPhi(z))(z\varphi(z))}{\varphi(z)^2}
\\&=\frac{\varphi(z+\alpha) - \varphi(z)+z(\varPhi(z+\alpha) - \varPhi(z))}{\varphi(z)}
\\&=\frac{\int_z^{z+\alpha}\varphi'(t)\DD t+z\int_z^{z+\alpha}\varPhi'(t)\DD t}{\varphi(z)}
\\&=\frac{\int_z^{z+\alpha}(-t\varphi(t))\DD t+z\int_z^{z+\alpha}\varphi(t)\DD t}{\varphi(z)}
\\&=-\frac{\int_z^{z+\alpha}(t-z)\varphi(t)\DD t}{\varphi(z)}<0
.}
Hence, $h(z)$ is a decreasing function. It follows from Lemma~\ref{LEM:0} that
\AL{&\frac{\varPhi(z+\alpha)-\varPhi(z)}{\varphi(z)}=h(z)\le h(0)
=\frac{\varPhi(\alpha)-\varPhi(0)}{\varphi(0)}=\sqrt{2\uppi}\,(\varPhi(\alpha)-\varPhi(0))
\\={}&\sqrt{2\uppi}\int_0^\alpha\varPhi'(t)\DD t=\sqrt{2\uppi}\int_0^\alpha\frac{\RM e^{-t^2/2}}{\sqrt{2\uppi}}\DD t\le\sqrt{2\uppi}\int_0^\alpha\frac1{\sqrt{2\uppi}}\DD t=\int_0^\alpha\DD t=\alpha.\qedhere}
\end{proof}

\begin{LEM}[a Gaussian inverse estimate]\label{LEM:4}
For every $0<v\le\frac12$,
\AL{\varphi(\varPhi^{-1}(1-v))\le v\sqrt{2\ln\frac1v}.}
\end{LEM}
\begin{proof}
Let $z:=\varPhi^{-1}(1-v)\ge0$, so that $v=1-\varPhi(z)=\BAR\varPhi(z)$.

Note that it is equivalent to show that
\AL{\ln\Big(\frac{1}{\BAR{\varPhi}(z)}\Big)\ge\frac{\varphi(z)^2}{2\BAR{\varPhi}(z)^2}.}
Define a function $h:\BB R_{\ge0}\to\BB R$ as
\AL{h(z):= \ln\Big(\frac{1}{\BAR{\varPhi}(z)}\Big) - \frac{\varphi(z)^2}{2\BAR{\varPhi}(z)^2},\qquad z\ge0.}
Since $z\ge 0$, then by \cite{gordon1941values},
\AL{\frac{\varphi(z)}{\BAR\varPhi(z)}\ge z\ge\frac z2.}
and by \cite{birnbaum1942inequality},
\AL{\frac{\varphi(z)}{\BAR\varPhi(z)}\le\frac2{\sqrt{z^2+4}-z}=\frac{\sqrt{z^2+4}+z}2=\frac z2+\frac{\sqrt{z^2+4}}2.}
Together, we have
\AL{\frac z2<\frac{\varphi(z)}{\BAR\varPhi(z)}\le\frac z2+\frac{\sqrt{z^2+4}}2.}
Furthermore, since $\BAR\varPhi'(z)=-\varphi(z)$, and $\varphi'(z)=-z\varphi(z)$,
\AL{
h'(z)&=\frac{\varphi(z)}{\BAR{\varPhi}(z)} \Big(1 + z\frac{\varphi(z)}{\BAR{\varPhi}(z)} -\Big(\frac{\varphi(z)}{\BAR{\varPhi}(z)}\Big)^2\Big)
\\&=\frac{\varphi(z)}{\BAR{\varPhi}(z)} \Big(\frac{\varphi(z)}{\BAR{\varPhi}(z)}-\frac z2+\frac{\sqrt{z^2+4}}2\Big)\Big(\frac z2+\frac{\sqrt{z^2+4}}2-\frac{\varphi(z)}{\BAR{\varPhi}(z)}\Big)
\\&\ge0
.}
Hence, the function $h(z)$ is non-decreasing w.r.t.\ $z$. It follows that for any $0<v\le\frac12$,
\AL{
\ln\Big(\frac{1}{v}\Big)-\frac{\varphi(\varPhi(1-v))^2}{2v^2}=\ln\Big(\frac{1}{\BAR{\varPhi}(z)}\Big)-\frac{\varphi(z)^2}{2\BAR{\varPhi}(z)^2}=h(z)\ge h(0)=\ln2-\frac1\uppi>0.
}
Therefore, $\varphi(\varPhi^{-1}(1-v))\le v\sqrt{2\ln\frac1v}$.\qedhere

\end{proof}

\begin{LEM}[a Gamma integral]\label{LEM:2}
Let $\Gamma(\beta)$ denote the Gamma function $(\beta>0)$. For any $\alpha,\beta>0$,
\AL{\int_0^1t^{\alpha-1}\Big(\!\ln\frac1t\Big)^{\beta-1}\!\DD t=\frac{\Gamma(\beta)}{\alpha^\beta}.}
In particular,
\AL{\int_0^1t\sqrt{\ln\frac1t}\DD t=\frac{\Gamma\big(\frac12+1\big)}{(1+1)^{\sfrac12+1}}=\frac{\sqrt\uppi}{4\sqrt2}.}
\end{LEM}

\begin{proof}
Let $z:=\alpha\ln\frac1t$. Then,
\AL{
\int_0^1t^{\alpha-1}\Big(\!\ln\frac1t\Big)^{\beta-1}\!\DD t={}&\int^{+\infty}_{0} \big(\RM e^{-z/\alpha}\big)^{\alpha-1}\Big(\frac z\alpha\Big)^{\beta-1}\frac{\RM e^{-z/\alpha}}{\alpha}\DD z
\\={}&\frac1{\alpha^{\beta}}\int^{+\infty}_{0}\RM e^{-z} z^{\beta-1}\DD z =\frac1{\alpha^{\beta}}\Gamma(\beta)
.\qedhere}
\end{proof}

\begin{LEM}[a mixed integral estimate]\label{LEM:3}
For every $\alpha>0$ and $0<\beta\le\alpha$,
\AL{
\int_0^\beta t\RM e^{-t}\sqrt{\ln\frac\alpha t}\DD t&\le\sqrt{\ln\alpha}\,(1-\RM e^{-\beta+\ln(\beta+1)})+\frac{\sqrt\uppi}{4\sqrt2}
.}
\end{LEM}

\begin{proof}
Since $\ln\frac1t\ge0$ only when $0\le t\le1$, then by the triangle inequality,
\AL{
&\sqrt{\ln\frac\alpha t}=\sqrt{\ln\alpha+\ln\frac1t}\le\sqrt{\ln\alpha+\BBM1_{[0\le t\le1]}\ln\frac1t}
\\\le{}&\sqrt{\ln\alpha}+\sqrt{\BBM1_{[0\le t\le1]}\ln\frac1t}=\sqrt{\ln\alpha}+\BBM1_{[0\le t\le1]}\sqrt{\ln\frac1t}
.}
It follows from Lemma~\ref{LEM:2} that
\AL{
\int_0^\beta t\RM e^{-t}\sqrt{\ln\frac\alpha t}\DD t\le{}&\int_0^\beta t\RM e^{-t}\Big(\sqrt{\ln\alpha}+\BBM1_{[0\le t\le1]}\sqrt{\ln\frac1t}\Big)\DD t
\\={}&\sqrt{\ln\alpha}\int_0^\beta t\RM e^{-t}\DD t+\int_0^{\min\{1,\beta\}} t\RM e^{-t}\sqrt{\ln\frac1t}\DD t
\\\le{}&\sqrt{\ln\alpha}\int_0^\beta t\RM e^{-t}\DD t+\int_0^{1} t\RM e^{-t}\sqrt{\ln\frac1t}\DD t
\\\le{}&\sqrt{\ln\alpha}\int_0^\beta t\RM e^{-t}\DD t+\int_0^{1} t\sqrt{\ln\frac1t}\DD t
\\={}&\sqrt{\ln\alpha}\,(1-\RM e^{-\beta+\ln(\beta+1)})+\frac{\sqrt\uppi}{4\sqrt2}
.\qedhere}
\end{proof}

\begin{LEM}[an integer function bound]\label{LEM:5}
For every integer $n\ge3$,
\AL{
\frac{\sqrt2}{(n-2)^2}\Big(\sqrt{\ln(n-2)}(1-\RM e^{-\frac{n}2-1+\ln\frac n2})+\frac{\sqrt\uppi}{4\sqrt2}\Big)
\le\frac32\sqrt{\frac{\uppi}{\ln3}}\frac{\sqrt{\ln n}}{n(n-1)}
.}
\end{LEM}

\begin{proof}
Define a function $h:\BB N_{\ge3}\to\BB R$:
\AL{h(n):=\frac{\frac{\sqrt2}{(n-2)^2}\big(\sqrt{\ln(n-2)}(1-\RM e^{-\frac{n}2-1+\ln\frac n2})+\frac{\sqrt\uppi}{4\sqrt2}\big)}{\frac{\sqrt{\ln n}}{n(n-1)}},\qquad n\ge3.}
Note that when $n\ge9$, we have
\AL{
h(n)&\le\frac{\frac{\sqrt2}{(n-2)^2}\big(\sqrt{\ln(n-2)}+\frac{\sqrt\uppi}{4\sqrt2}\big)}{\frac{\sqrt{\ln n}}{n(n-1)}}=\frac{\frac{\sqrt2}{(n-2)^2}\big(\sqrt{\ln(n-2)}+\frac{\sqrt\uppi}{4\sqrt2}\big)}{\frac{\sqrt{\ln n}}{n(n-1)}}
\\&=\Big(\sqrt2\sqrt{\frac{\ln(n-2)}{\ln n}}+\frac14\sqrt{\frac{\uppi}{\ln n}}\Big)\Big(1+\frac2{n-2}+\frac3{(n-2)^2}\Big)
\\&\le\Big(\sqrt2+\frac14\sqrt{\frac{\uppi}{\ln n}}\Big)\Big(1+\frac2{n-2}+\frac3{(n-2)^2}\Big)
\\&\le\Big(\sqrt2+\frac14\sqrt{\frac{\uppi}{\ln 9}}\Big)\Big(1+\frac2{9-2}+\frac3{(9-2)^2}\Big)
\\&=\Big(\sqrt2+\frac14\sqrt{\frac{\uppi}{\ln 9}}\Big)\frac{72}{49}
<2.52.}
It follows that
\AL{
&\frac{\frac{\sqrt2}{(n-2)^2}\big(\sqrt{\ln(n-2)}+\frac{\sqrt\uppi}{4\sqrt2}\big)}{\frac{\sqrt{\ln n}}{n(n-1)}}=h(n)
\\\le{}&\!\max\!\Big\{h(3),h(4),\dots,h(8),\Big(\sqrt2+\frac14\sqrt{\frac{\uppi}{\ln 9}}\Big)\frac{72}{49}\Big\}
\\={}&h(3)=\frac32\sqrt{\frac{\uppi}{\ln3}}<2.54
.\qedhere}
\end{proof}

\begin{LEM}[a Gaussian integral estimate]\label{LEM:6}
For every integer $n\ge3$,
\AL{\int_0^{+\infty}\!\!\!\!\!\!\varPhi(z)^{n-2}\varphi(z)^2\DD z\le\frac32\sqrt{\frac{\uppi}{\ln3}}\frac{\sqrt{\ln n}}{n(n-1)}
.}
\end{LEM}
\begin{proof}
Let $v:=1-\varPhi(z)$ and $t:=(n-2)v$. By the fact that $1-v\le\RM e^{-v}$ and Lemmas~\ref{LEM:4}, \ref{LEM:3}, \& \ref{LEM:5},
\AL{
\int_0^{+\infty}\!\!\!\!\!\!\varPhi(z)^{n-2}\varphi(z)^2\DD z={}&\int_0^{1/2}\!\!\!\!(1-v)^{n-2}\varphi(\varPhi^{-1}(1-v))\DD v
\le\int_0^{1/2}\!\!\!\!(\RM e^{-v})^{n-2}\varphi(\varPhi^{-1}(1-v))\DD v
\\\le{}&\int_0^{1/2}\!\!\!\!(\RM e^{-v})^{n-2}v\sqrt{2\ln\frac1v}\DD v
=\frac{\sqrt2}{(n-2)^2}\int_0^{\frac{n}2-1}\!\!\!\!\!\! t\RM e^{-t}\sqrt{\ln\frac{n-2}t}\DD t
\\\le{}&\frac{\sqrt2}{(n-2)^2}\Big(\sqrt{\ln(n-2)}(1-\RM e^{-\frac{n}2-1+\ln\frac n2})+\frac{\sqrt\uppi}{4\sqrt2}\Big)
\\\le{}&\frac32\sqrt{\frac{\uppi}{\ln3}}\frac{\sqrt{\ln n}}{n(n-1)}<2.54\frac{\sqrt{\ln n}}{n(n-1)}
.\qedhere}
\end{proof}

With the technical lemmata above, we are now ready to prove Theorem~\ref{THM:imb}.

\begin{proof}[Proof of Theorem~\ref{THM:imb}]
Let $\BM\xi:=\BM P^{(l)}\BM x^{(l)}$ denote the logits of routing weights, so that $\BM\pi^{(l)}=\OP{softmax}(\BM\xi)$. Let $\xi_{(1)}\ge\dots\ge\xi_{(n)}$ denote the order statistics of $\BM\xi$ (i.e., $\xi_{(1)}$ is the largest entry of $\BM\xi$, $\xi_{(2)}$ is the second largest entry of $\BM\xi$, etc.). Note that
\AL{
\OP{ESS}(\BM\pi^{(l)})={}&\frac{\|\BM\pi^{(l)}\|_1^2}{\|\BM\pi^{(l)}\|_2^2}=\frac{\big(\!\sum_{i=1}^n\pi_i^{(l)}\big)^2}{\sum_{i=1}^n(\pi^{(l)}_i)^2}=\frac{\big(\!\sum_{i=1}^n\OP{softmax}(\BM\xi)_i\big)^2}{\sum_{i=1}^n\OP{softmax}(\BM\xi)_i^2}
\\={}&\frac{\big(\!\sum_{i=1}^n\RM e^{\xi_i}\big)^2}{\sum_{i=1}^n(\RM e^{\xi_{i}})^2}=\frac{\big(\!\sum_{i=1}^n\RM e^{\xi_{(i)}}\big)^2}{\sum_{i=1}^n(\RM e^{\xi_{(i)}})^2}\le\frac{\big(\!\sum_{i=1}^n\RM e^{\xi_{(i)}}\big)^2}{(\RM e^{\xi_{(1)}})^2}
\\={}&\bigg(\!1+\sum_{i=2}^n\frac1{\RM e^{\xi_{(1)}-\xi_{(i)}}}\!\bigg)^{\!2}\le\bigg(\!1+\sum_{i=2}^n\frac1{\RM e^{\xi_{(1)}-\xi_{(2)}}}\!\bigg)^{\!2}
\\={}&\Big(1+\frac{n-1}{\RM e^{\xi_{(1)}-\xi_{(2)}}}\Big)^{\!2}=\Big(1+\frac1{\RM e^{\xi_{(1)}-\xi_{(2)}-\ln(n-1)}}\Big)^{\!2}
.}
Since $\BM P^{(l)}$ have i.i.d.\ $\CAL N(0,\sigma^2)$ entries, then $\BM\xi=\BM P^{(l)}\BM x^{(l)}$ have i.i.d\ $\CAL N(0,\sigma^2\|\BM x^{(l)}\|_2^2)$ entries. Let
\AL{\kappa:=\frac{1}{\frac32\sqrt{\frac{\uppi}{\ln3}\ln n}+\frac1{\sqrt{2\uppi}\,2^{n-\log_2n-1}}}.}
For any $0<\delta<1$, with $z_{(i)}:=\frac{\xi_{(i)}-0}{\sigma\|\BM x^{(l)}\|_2}$ ($i=1,2$), by Lemmas~\ref{LEM:0}, \ref{LEM:1}, \& \ref{LEM:6},
\AL{
&\Prb[\xi_{(1)}-\xi_{(2)}\le\delta\kappa\sigma\|\BM x^{(l)}\|_2]=\Prb[z_{(1)}-z_{(2)}\le\delta\kappa]
\\={}&\int_{-\infty}^{+\infty}\!\!\!\int_{z_{(2)}}^{z_{(2)}+\delta\kappa}\!\!\!\!\!\!\!\!\!\!\!n(n-1)\varphi(z_{(1)})\varphi(z_{(2)})\varPhi(z_{(2)})^{n-2}\DD{z_{(1)}}\DD{z_{(2)}}
\\={}&n(n-1)\!\int_{-\infty}^{+\infty}\!\!\!\int_{z_{(2)}}^{z_{(2)}+\delta\kappa}\!\!\!\!\!\!\!\!\!\!\!\varphi(z_{(1)})\DD{z_{(1)}}\varphi(z_{(2)})\varPhi(z_{(2)})^{n-2}\DD{z_{(2)}}
\\={}&n(n-1)\!\int_{-\infty}^{+\infty}\!\!\!(\varPhi(z_{(2)}+\delta\kappa)-\varPhi(z_{(2)}))\varphi(z_{(2)})\varPhi(z_{(2)})^{n-2}\DD{z_{(2)}}
\\={}&n(n-1)\bigg(\!\int_{-\infty}^0+\int_0^{+\infty}\!\bigg)(\varPhi(z_{(2)}+\delta\kappa)-\varPhi(z_{(2)}))\varphi(z_{(2)})\varPhi(z_{(2)})^{n-2}\DD{z_{(2)}}
\\\le{}&n(n-1)\bigg(\!\int_{-\infty}^0\!\!\frac{\delta\kappa}{\sqrt{2\uppi}}\varphi(z_{(2)})\varPhi(z_{(2)})^{n-2}\DD{z_{(2)}}+\int_0^{+\infty}\!\!\!\!\delta\kappa\varphi(z_{(2)})\varphi(z_{(2)})\varPhi(z_{(2)})^{n-2}\DD{z_{(2)}}\!\!\bigg)
\\={}&\delta\kappa n(n-1)\bigg(\!\frac1{\sqrt{2\uppi}}\!\int_{-\infty}^0\!\!\varphi(z_{(2)})\varPhi(z_{(2)})^{n-2}\DD{z_{(2)}}+\int_0^{+\infty}\!\!\!\!\varPhi(z_{(2)})^{n-2}\varphi(z_{(2)})^2\DD{z_{(2)}}\!\!\bigg)
\\={}&\delta\kappa n(n-1)\bigg(\!\frac1{\sqrt{2\uppi}}\frac{\varPhi(0)^{n-1}-\varPhi(-\infty)^{n-1}}{n-1}+\int_0^{+\infty}\!\!\!\!\varPhi(z_{(2)})^{n-2}\varphi(z_{(2)})^2\DD{z_{(2)}}\!\!\bigg)
\\={}&\delta\kappa n(n-1)\bigg(\!\frac1{\sqrt{2\uppi}(n-1)2^{n-1}}+\int_0^{+\infty}\!\!\!\!\varPhi(z_{(2)})^{n-2}\varphi(z_{(2)})^2\DD{z_{(2)}}\!\!\bigg)
\\\le{}&\delta\kappa n(n-1)\bigg(\!\frac1{\sqrt{2\uppi}(n-1)2^{n-1}}+\frac32\sqrt{\frac{\uppi}{\ln3}}\frac{\sqrt{\ln n}}{n(n-1)}\!\bigg)
\\={}&\delta\kappa\bigg(\frac32\sqrt{\frac{\uppi}{\ln3}\ln n}+\frac n{\sqrt{2\uppi}\,2^{n-1}}\bigg)
=\delta\kappa\bigg(\frac32\sqrt{\frac{\uppi}{\ln3}\ln n}+\frac1{\sqrt{2\uppi}\,2^{n-\log_2n-1}}\bigg)=\delta
.}
This implies $\Prb[\xi_{(1)}-\xi_{(2)}>\delta\kappa\sigma\|\BM x^{(l)}\|_2]\ge1-\delta$. It follows that with probability at least $1-\delta$,
\AL{
\OP{ESS}(\BM\pi^{(l)})\le{}&\Big(1+\frac1{\RM e^{\xi_{(1)}-\xi_{(2)}-\ln(n-1)}}\Big)^{\!2}\le\Big(1+\frac1{\RM e^{\delta\kappa\sigma\|\BM x^{(l)}\|_2-\ln(n-1)}}\Big)^2
\\={}&\left(1+\frac1{\exp\!\left(\!\dfrac{\delta\sigma\|\BM x^{(l)}\|_2}{\frac32\sqrt{\frac{\uppi}{\ln3}\ln n}+\frac1{\sqrt{2\uppi}\,2^{n-\log_2n-1}}}-\ln(n-1)\!\right)}\!\right)^{\!\!2}
.\qedhere}
\end{proof}

\subsection{Proof of Theorem~\ref{THM:topk}}\label{app:theo-topk}

Before stating our proof of Theorem~\ref{THM:imb}, we present a technical lemma that we will employ.

To simplify notation, we omit the superscript $^{(l)}$ in this proof. For an ordered subset $\CAL I=(i_1,\dots,i_k)\subseteq\{1,\dots,n\}$, let $q(\CAL I)$ denote the probability of sampling an ordered subset $\CAL I$ from $\BM q$ without replace:
\AL{Q(\CAL I)=Q(i_1,\dots,i_k):=\prod_{j=1}^k\frac{q_{i_j}}{1-\sum_{j'=1}^{j-1}q_{i_{j'}}}.}
Let $\CAL P_k$ denote the set of permutations over $\{1,\dots,n\}$. For $\varpi\in\CAL P_k$, define the permutation action as $\varpi(i_1,\dots,i_k):=(i_{\varpi(1)},\dots,i_{\varpi(k)})$. Let $Q(\CAL I)$ denote the probability of sampling an unordered subset $\CAL I$ from $\BM q$ without replacement:
\AL{\BAR Q(\CAL I)=\Prb_{\CAL I\sim\BM q}[\CAL I]=\sum_{\varpi\in\CAL P_n}Q(\varpi(\CAL I)).}

\begin{LEM}[swapping a pair]\label{LEM:swap}
Given a size-$k$ subset $\CAL I\subseteq\{1,\dots,n\}$, for a LoRA $i\in\CAL I$ and another LoRA $i^\dagger\in\{1,\dots,n\}\setminus\CAL I$, if $q_i\le q_{i^\dagger}$, then replacing $i$ with $i^\dagger$ increases the unordered sampling probability:
\AL{\BAR Q((\CAL I\setminus\{i\})\cup\{i^\dagger\})\ge\BAR Q(\CAL I).}
\end{LEM}

\begin{proof}
Say $\CAL I=(i_1,\dots,i_k)$. Without loss of generality, say $i_1=i$, and let $\CAL I^\dagger:=(i^\dagger,i_2,\dots,i_k)$ denote the ordered subset after replacing $i$ with $i^\dagger$. For any permutation $\varpi\in\CAL P_k$, let $j_\varpi:=\varpi^{-1}(1)$ denote the order of $i$ under permutation $\varpi$ (i.e., $\varpi(\CAL I)_{j_\varpi}=i$). Since $q_i\le q_{i^\dagger}$, then
\AL{
\frac{Q(\varpi(\CAL I^\dagger))}{Q(\varpi(\CAL I))}={}&\frac{q_{i^\dagger}}{q_i}\prod_{j=j_\varpi+1}^k\frac{1-\sum_{j'=1}^{j-1}q_{i_{j'}}}{1-q_{i^\dagger}+q_i-\sum_{j'=1}^{j-1}q_{i_{j'}}}
\\={}&\frac{q_{i^\dagger}}{q_i}\prod_{j=j_\varpi+1}^k\frac1{1-\frac{q_{i^\dagger}-q_i}{1-\sum_{j'=1}^{j-1}q_{i_{j'}}}}
\\\ge{}&\frac{q_{i^\dagger}}{q_i}\prod_{j=j_\varpi+1}^k1=\frac{q_{i^\dagger}}{q_i}\ge1
.}
This means $Q(\varpi(\CAL I^\dagger))\ge Q(\varpi(\CAL I))$. It follows that
\AL{
\BAR Q((\CAL I\setminus\{i\})\cup\{i^\dagger\}){}
&=\BAR Q(\CAL I^\dagger)=\sum_{\varpi\in\CAL P_n}Q(\varpi(\CAL I^\dagger))
\\&\ge\sum_{\varpi\in\CAL P_n}Q(\varpi(\CAL I))=\BAR Q(\CAL I)
.\qedhere}
\end{proof}

We are now ready to prove Theorem~\ref{THM:topk}.

\begin{proof}[Proof of Theorem~\ref{THM:topk}]
Suppose that
\AL{\CAL I^\dagger:=\mathop{\RM{argtop}_k}_{i=1}^nq_i\ne\CAL I^*,}
where we break ties in $\RM{argtop}$ arbitrarily.
We will show that this premise leads to a contradiction.

Recall that by definition,
\AL{\BAR Q(\CAL I^*)=\Prb_{\CAL I\sim\BM q}[\CAL I=\CAL I^{*}]>\frac12.}
Since $\CAL I^\dagger\ne\CAL I^*$, then $k^\cap:=|\CAL I^*\cap\CAL I^\dagger|<k$. Say $\CAL I^*\setminus\CAL I^\dagger=\{i^*_1,\dots,i^*_{k-k^\cap}\}$, $\CAL I^\dagger\setminus\CAL I^*=\{i^\dagger_1,\dots,i^\dagger_{k-k^\cap}\}$. Construct a series of subsets inductively as follows. Define $\TLD{\CAL I}_0:=\CAL I^*$. For $j=1,\dots,k-k^\cap$, define $\TLD{\CAL I}_j$ by replacing $i_j^*$ from $\TLD{\CAL I}_{j-1}$ with $i_j^\dagger$ and inheriting all other LoRAs from $\TLD{\CAL I}_{j-1}$. Finally, we have $\TLD{\CAL I}_{k-k^\cap}=\CAL I^\dagger$. Since $\CAL I^\dagger$ consists of LoRAs $i$ with top-$k$ $q_i$, then $q_{i_j^*}\le q_{i_j^\dagger}$ for all $j=1,\dots,k-k^\cap$. Hence, by Lemma~\ref{LEM:swap}, $\BAR Q(\TLD{\CAL I}_j)\ge\BAR Q(\TLD{\CAL I}_{j-1})$ for all $j=1,\dots,k-k^\cap$. Together,
\AL{
\BAR Q(\CAL I^\dagger)=\BAR Q(\TLD{\CAL I}_{k-k^\cap})\ge\BAR Q(\TLD{\CAL I}_{k-k^\cap-1})\ge\cdots\ge\BAR Q(\TLD{\CAL I}_{0})=\BAR Q(\CAL I^*)>\frac12
.}
It follows that
\AL{\BAR Q(\CAL I^\dagger)+\BAR Q(\CAL I^*)>\frac12+\frac12=1.}
However, this contradicts the fact that
\AL{\BAR Q(\CAL I^\dagger)+\BAR Q(\CAL I^*)\le\sum_{\CAL I}\BAR Q(\CAL I)=1,}
falsifying the premise. Therefore,
\AL{&\mathop{\RM{argtop}_k}_{i=1}^nq_i=\CAL I^*.\qedhere}
\end{proof}

\end{document}